%% file: main.tex
\documentclass{article}
\usepackage{iclr2027_conference,times}

\usepackage{amsmath,amsfonts,amssymb,amsthm,bm}
\usepackage{booktabs}
\usepackage{graphicx}
\usepackage[table]{xcolor}
\usepackage{multirow}
\usepackage{algorithm}
\usepackage[noend]{algpseudocode}
\usepackage{hyperref}
\usepackage{url}

\input{math_commands.tex}

\input{numbers.tex}

\definecolor{lamaccent}{HTML}{A6CE39}
\newcommand{\best}[1]{\cellcolor{lamaccent!85}\textbf{#1}}
\newcommand{\shiprow}{\rowcolor{lamaccent!35}}

\algrenewcommand\algorithmicrequire{\textbf{Input:}}
\algrenewcommand\algorithmicensure{\textbf{Output:}}
\algrenewcommand\alglinenumber[1]{\scriptsize#1:}
\makeatletter
\renewcommand{\ALG@beginalgorithmic}{\footnotesize}
\makeatother

\newtheorem{theorem}{Theorem}
\newtheorem{corollary}{Corollary}
\newtheorem{proposition}{Proposition}
\theoremstyle{definition}
\newtheorem{definition}{Definition}

\title{\texorpdfstring{{\setlength{\fboxsep}{2pt}\colorbox{lamaccent!85}{LAM}}}{LAM}:
A Lossy Agent Memory Framework With a Retrieval-Score Error Bound}

\newif\ifarxiv
\arxivtrue      % <-- \arxivtrue = arXiv (named) | \arxivfalse = review

\ifarxiv\iclrfinalcopy\fi

\ifarxiv
\newcommand{\pairgap}{4em}
\newcommand{\pairrow}[6]{%
  \makebox[\dimexpr\textwidth-2\tabcolsep][c]{%
    \begin{tabular}[t]{@{}l@{\hskip\pairgap}l@{}}
    #1 & #4 \\ \mdseries #2 & \mdseries #5 \\ \mdseries #3 & \mdseries #6
    \end{tabular}}}
\author{%
\pairrow{Baixi Sun}{Argonne National Laboratory}{\texttt{bsun@anl.gov}}
        {Le Chen}{Argonne National Laboratory}{\texttt{lechen@anl.gov}}
\AND
\pairrow{Anjir Ahmed Chowdhury}{University of Houston}{\texttt{aachowd4@CougarNet.UH.EDU}}
        {Xiaolong Ma}{Argonne National Laboratory}{\texttt{xma@anl.gov}}
\AND Chih-Hsuan Yang \\ Argonne National Laboratory \\ \texttt{bellayang@anl.gov}
\And Mingze Xia \\ Oregon State University \\ \texttt{xiami@oregonstate.edu}
\And 
Syed Zawad \\ IBM Research \\ \texttt{szawad@ibm.com}
\AND Sheng Di \\ Argonne National Laboratory \\ \texttt{sdi1@anl.gov}
\And Rajkumar Kettimuthu \\ Argonne National Laboratory \\ \texttt{kettimut@mcs.anl.gov}
\And
 Huihuo Zheng \\ Argonne National Laboratory \\ \texttt{huihuo.zheng@anl.gov}
\AND Rajeev Thakur \\ Argonne National Laboratory \\ \texttt{thakur@anl.gov}
\And Venkatram Vishwanath \\ Argonne National Laboratory \\ \texttt{vvishwanath@anl.gov}
\And
 Feng Yan \\ University of Houston \\ \texttt{fyan5@Central.UH.EDU}
}
\else
\author{Anonymous authors\\
Paper under double-blind review}
\fi

\begin{document}

\maketitle

% \iclrfinalcopy sets the running head to "Published as a conference paper at
% ICLR 2027" (sty:88), which on a preprint is a false claim. It is set inside
% \@maketitle, so it can only be overridden after \maketitle has run -- here.
% \lhead is fancyhdr's and is called from this file, so the style file stays
% untouched. IF THE PAPER IS ACCEPTED: delete these two lines and the
% template's own header becomes correct again.
\ifarxiv\lhead{Preprint. Under review.}\fi

\begin{abstract}
\input{sections/00-abstract}
\end{abstract}

% --- body: 1.5 + 3 + 3 + 1.3 + 0.2 pages ------------------------------------
\input{sections/01-intro}
\input{sections/04-method}
\input{sections/08-eval}
\input{sections/02-related}
\input{sections/11-conclusion}

% The statements are exempt from the page limit and are read on their own, so
% they get their own page rather than sharing one with the conclusion or the
% bibliography.
\clearpage
\input{sections/10-statements}

\clearpage
\bibliography{refs}
\bibliographystyle{iclr2027_conference}

% --- appendix ---------------------------------------------------------------
% Everything the page budget pushed out of the body lives here. The body cites
% it by \ref, so these inputs are load-bearing: dropping one reintroduces
% undefined references rather than merely losing material.
\clearpage
\appendix
\input{sections/impl}
\input{sections/03-redundancy}
\input{sections/05-theory}
\input{sections/06-determinism}
\input{sections/07-cost}
\input{sections/B-eval-extended}
\input{sections/A-appendix}
\input{sections/I-cost}

\end{document}

%% file: math_commands.tex
\usepackage{amsmath,amsfonts,bm}

\def\eqref#1{equation~\ref{#1}}
\def\1{\bm{1}}

\DeclareMathAlphabet{\mathsfit}{\encodingdefault}{\sfdefault}{m}{sl}
\SetMathAlphabet{\mathsfit}{bold}{\encodingdefault}{\sfdefault}{bx}{n}

%% file: numbers.tex
\newcommand{\hwBW}{237.4\,GB/s}          % triad, 3 x 512 MB fp32
\newcommand{\hwMUm}{0.869}               % of the 273 GB/s spec
\newcommand{\hwFLOPS}{90.3\,TFLOP/s}     % fp16, 8192^3 GEMM
\newcommand{\hwBalance}{380}             % FLOP/byte
\newcommand{\encRate}{71{,}708}          % tok/s, bge-base, SYNTHETIC text
\newcommand{\quantRate}{65.4\,GB/s}      % Q_eps, 2^26 fp32, unfused
\newcommand{\hashRate}{38.2\,GB/s}       % xxHash3, CPU, 64 MB
\newcommand{\lTwoRate}{13.0}             % M pairs/s at D=768
\newcommand{\cuszpRate}{53.3\,GB/s}
\newcommand{\cuszpCR}{5.18}
\newcommand{\bytesPerTok}{3.68}          % B/token, SWE-smith tool split

\newcommand{\corpusRecords}{361{,}802}
\newcommand{\corpusTraj}{6{,}026}
\newcommand{\corpusTokens}{136.3M}
\newcommand{\subTraj}{600}
\newcommand{\subObs}{15{,}169}
\newcommand{\subChunks}{27{,}986}
\newcommand{\subTokens}{9.07M}

\newcommand{\exRecord}{1.61\%}
\newcommand{\exPara}{1.61\%}
\newcommand{\exLine}{18.14\%}
\newcommand{\exChunk}{2.61\%}

\newcommand{\nearGreedyNineFive}{13.81\%}
\newcommand{\nearGreedyNineNine}{6.52\%}
\newcommand{\lineAlone}{18.04\%}         % free line-exact hashing, same denominator
\newcommand{\jointGross}{23.19\%}        % greedy near -> line, online, gross
\newcommand{\jointNet}{22.50\%}          % net of the reference stubs we pay for
\newcommand{\stubCost}{0.69\%}           % of input, spent on stubs
\newcommand{\jointGainNet}{4.47}         % pts over free line-exact, net
\newcommand{\ctxGainOurs}{1.290$\times$} % 1/(1-r), net
\newcommand{\ctxGainLine}{1.220$\times$}
\newcommand{\ctxGainNear}{1.160$\times$} % near-dup layer alone
\newcommand{\ctxGainGross}{1.302$\times$}
\newcommand{\ctxGainDelta}{1.058$\times$} % ours net / free line-exact
\newcommand{\costPerEventOurs}{8.88\,ms}
\newcommand{\costPerEventHash}{3.08\,$\mu$s}
\newcommand{\turnRefCost}{1.003\,s}      % the turn these per-event costs sit inside
\newcommand{\indepOverstate}{4.0--11.7}  % pts, the estimate to never report (greedy rule)

\newcommand{\locomoPairs}{2}
\newcommand{\locomoTotal}{17{,}296{,}021}

\newcommand{\turnDelta}{637}             % tokens added per turn
\newcommand{\turnFobs}{94\%}             % of which tool observation
\newcommand{\turnGa}{39}                 % tokens DECODED per turn

\newcommand{\phiSmall}{2.5\%}            % batch 1,  C = 3K
\newcommand{\phiOne}{21.6\%}             % batch 1,  C = 32K
\newcommand{\phiThirtyTwo}{76.6\%}       % batch 32, C = 32K
\newcommand{\phiBig}{92.9\%}             % batch 32, C = 128K
\newcommand{\ceilSmall}{1.03$\times$}
\newcommand{\ceilOne}{1.28$\times$}
\newcommand{\ceilThirtyTwo}{4.28$\times$}
\newcommand{\ceilBig}{14.1$\times$}

\newcommand{\polDelete}{1.186\times10^{6}}   % s, delete-in-place, batch 1
\newcommand{\polAdmit}{1.66\times10^{4}}     % s, eps error-bounded admission, batch 1
\newcommand{\polDeleteBs}{1.183\times10^{6}} % s, delete-in-place, batch 32
\newcommand{\polAdmitBs}{1.29\times10^{4}}   % s, admission, batch 32
\newcommand{\polGainOne}{71.4$\times$}
\newcommand{\polGainBs}{91.6$\times$}
\newcommand{\polShareDelete}{98.6\%}     % compaction as share of total runtime
\newcommand{\polShareAdmit}{0.1\%}
\newcommand{\reprefillCost}{538}         % s, re-prefilling a 360K-token KV cache
\newcommand{\turnServeCost}{12}          % s, the turn that re-prefill serves

\newcommand{\invalMean}{0.94}          % mean fraction of context invalidated
\newcommand{\invalMedian}{0.98}        % median
\newcommand{\invalGeNF}{76\%}          % trajectories invalidating >= 0.95
\newcommand{\invalFirstDup}{4.8\%}     % median position of the first duplicate
\newcommand{\invalNTraj}{300}          % trajectories measured
\newcommand{\polGainOneMeas}{63.5$\times$}
\newcommand{\polGainBsMeas}{81.5$\times$}

\newcommand{\turnLatMean}{18.1\,ms}
\newcommand{\turnLatPnn}{115\,ms}
\newcommand{\turnCostLo}{306\,ms}
\newcommand{\turnCostHi}{3.3\,s}
\newcommand{\onlineEqBatch}{600/600}
\newcommand{\deltaAttained}{0.268}
\newcommand{\deltaConfigured}{0.316}
\newcommand{\cosRepro}{$2.4\times10^{-7}$}
\newcommand{\deltaSpread}{$4.9\times10^{-4}$}

\newcommand{\growthTurnsOurs}{1{,}727}   % turns before the KV cap stops the run
\newcommand{\growthTurnsLine}{1{,}615}
\newcommand{\growthGain}{1.07$\times$}   % ours vs free line-exact, growth regime
\newcommand{\budgetGain}{1.00$\times$}   % under a hard context budget
\newcommand{\meanCtxSnapshot}{2.8K}      % what the withdrawn model predicted
\newcommand{\meanCtxArrival}{28K}        % what the corrected model gives
\newcommand{\kvCap}{854K}                % tokens, Llama-3.1-8B at batch 1 on GB10

\newcommand{\growthCurveTraj}{600}
\newcommand{\growthSmallT}{1K}           % input size at the left of Fig.3(b)
\newcommand{\growthSmallRem}{0.80\%}     % ours, cos >= 0.95, at T = 1K  (n = 504)
\newcommand{\growthLargeT}{33K}          % rightmost T with n(T) >= 25
\newcommand{\growthLargeRem}{25.58\%}    % ours, cos >= 0.95, at T = 33K (n = 36)

\newcommand{\growthWholeRem}{22.50\%}    % whole-trajectory mean, matches \safOursRemoval
\newcommand{\maskOverOurs}{9.4$\times$}  % masking's delete-in-place vs our admission
\newcommand{\eeTurns}{1{,}453}
\newcommand{\eeNoneTurns}{1{,}341}      % turn at which no-compaction exhausts KV
\newcommand{\eeMaskOne}{9.4$\times$}
\newcommand{\eeMaskBs}{11.8$\times$}
\newcommand{\eeSummOne}{43.6$\times$}
\newcommand{\eeSummBs}{55.8$\times$}
\newcommand{\eeLinguaOne}{69.1$\times$}
\newcommand{\eeLinguaBs}{88.7$\times$}
\newcommand{\eeSdOne}{56.7$\times$}
\newcommand{\eeSdBs}{72.7$\times$}
\newcommand{\eeLzeroOne}{1.1$\times$}
\newcommand{\eeLzeroBs}{1.1$\times$}
\newcommand{\eeSecMask}{107.7}
\newcommand{\eeSecMaskBs}{105.1}
\newcommand{\eeSecSumm}{498.0}
\newcommand{\eeSecSummBs}{495.5}
\newcommand{\eeSecSd}{647.9}
\newcommand{\eeSecSdBs}{645.4}
\newcommand{\eeSecLingua}{790.5}
\newcommand{\eeSecLinguaBs}{787.9}
\newcommand{\eeSecLzero}{12.0}
\newcommand{\eeSecLzeroBs}{9.5}
\newcommand{\eeSecOurs}{11.4}
\newcommand{\eeSecOursBs}{8.9}
\newcommand{\eeReffMask}{0.75}
\newcommand{\eeReffSumm}{0.90}
\newcommand{\eeReffLingua}{0.24}
\newcommand{\eeReffSd}{0.32}
\newcommand{\eeReffLzero}{0.17}
\newcommand{\eeReffOurs}{0.22}

\newcommand{\compVsLlama}{4{,}957$\times$}
\newcommand{\compVsLingua}{1.33$\times$}
\newcommand{\compVsSemDeDup}{1.00$\times$}
\newcommand{\oursNoEncode}{1.49\,ms}     % T = 10^6, embeddings already exist
\newcommand{\encodeShare}{99.989\%}      % of our own cost at T = 10^6
\newcommand{\hashFloor}{0.0963\,ms}      % exact-hash floor at T = 10^6
\newcommand{\vsHashFloor}{15.4$\times$}  % we remain above it
\newcommand{\costRatioHash}{2{,}883$\times$} % our per-event cost vs exact hashing

\newcommand{\safTraj}{598}               % of 600; 2 witness no region line
\newcommand{\safJoined}{600}             % instance_id join against SWE-smith
\newcommand{\safLines}{6{,}276}          % witnessed evidence lines, pooled
\newcommand{\safPlusSide}{99.2\%}        % patch-direction check, pre-first-edit
\newcommand{\safMinusSide}{4.2\%}
\newcommand{\safOursRemoval}{22.47\%}    % cos >= 0.95, the shipped threshold
\newcommand{\safOursMicro}{99.984\%}

\newcommand{\safOursLost}{1}             % evidence lines lost, of \safLines

\newcommand{\safRemovalRatio}{5.1$\times$}  % ours / mask-50 removal
\newcommand{\safLossRatio}{54$\times$}      % mask-50 / ours evidence lost
\newcommand{\safMaskTwentyRemoval}{34.06\%} % matched-removal comparison
\newcommand{\safMaskTwentyMicro}{93.021\%}
\newcommand{\safOursNinetyRemoval}{31.99\%}
\newcommand{\safOursNinetyMicro}{98.709\%}
\newcommand{\safOursNineNineMicro}{100.000\%}
\newcommand{\safOursNineEightMicro}{100.000\%}
\newcommand{\safOursNineSixMicro}{100.000\%}
\newcommand{\safStubTouched}{2{,}005}    % evidence lines the near layer suppressed
\newcommand{\safStubRecovered}{2{,}004}  % still readable elsewhere in the residual
\newcommand{\safStubLost}{1}             % the reference-stub case, measured
\newcommand{\safCtrlOther}{0.7\%}        % specificity: 42/6,276 found in an unrelated trajectory
\newcommand{\safResolvedTraj}{454}       % resolved-only slice
\newcommand{\safResolvedRemoval}{18.61\%}
\newcommand{\safResolvedMicro}{99.979\%}

\newcommand{\embedDim}{768}
\newcommand{\ceilMaxCos}{0.9899}  % max cos between two survivors, t=0.99 run
\newcommand{\ceilAttainedNN}{8.42\%}\newcommand{\ceilCeilingNN}{14.16\%}\newcommand{\ceilFracNN}{59.48\%}
\newcommand{\ceilAttainedNF}{15.86\%}\newcommand{\ceilCeilingNF}{61.19\%}\newcommand{\ceilFracNF}{25.92\%}
\newcommand{\ceilAttainedN}{28.35\%}\newcommand{\ceilCeilingN}{91.86\%}\newcommand{\ceilFracN}{30.87\%}
\newcommand{\ceilChunks}{27{,}986}  % |S|, chunks presented corpus-wide
\newcommand{\ceilTokAttNN}{6.16\%}\newcommand{\ceilTokCeilNN}{11.47\%}
\newcommand{\ceilTokAttNF}{13.13\%}\newcommand{\ceilTokCeilNF}{62.55\%}
\newcommand{\ceilTokAttN}{26.12\%}\newcommand{\ceilTokCeilN}{91.63\%}
\newcommand{\costTokens}{9{,}017{,}033}  % observation tokens, 598 trajectories
\newcommand{\costDelta}{636.8}           % measured tokens per agent turn
\newcommand{\costSavedPerK}{3.39}        % M context tokens saved per 1k traj
\newcommand{\costAmp}{13.1$\times$}      % n/2 re-read amplification, n=26.28
\newcommand{\ecrOneOurs}{21.49\%}        % ECR@1.000, at cos 0.96; 0 lines lost
\newcommand{\ecrNNNOurs}{22.47\%}        % ECR@0.999
\newcommand{\ecrNNOurs}{22.47\%}         % ECR@0.99
\newcommand{\ecrNNMask}{4.42\%}
\newcommand{\ecrNNRatio}{5.1$\times$}
\newcommand{\ecrNFOurs}{31.99\%}         % ECR@0.95

\newcommand{\breakevenZero}{6.5}         % kappa where mask5 nets nothing
\newcommand{\breakevenTraj}{15}          % same, per-trajectory repair accounting
\newcommand{\safMaskFiveRemoval}{77.94\%}  % masking k=5, for the balance section
\newcommand{\turnsMean}{26.28}             % mean turns per trajectory

\newcommand{\safOursShipLines}{1}        % ours, t=0.95: lines lost of \safLines
\newcommand{\safOursShipLoss}{0.016\%}   % 1 - micro_safety, ours95
\newcommand{\safMaskTenLoss}{19.2\%}     % mask k=10
\newcommand{\safMaskTwentyLoss}{6.98\%}  % mask k=20, the last drawn point
\newcommand{\lossClip}{10\%}            % where panel (d) stops resolving
\newcommand{\safMaskFiveLoss}{26.9\%}    % mask k=5
\newcommand{\lossIntolerable}{20\%}      % STATED requirement, not measured
\newcommand{\bfrNoneA}{22.0\%}\newcommand{\bfrNoneB}{67.2\%}
\newcommand{\bfrNoneC}{93.2\%}\newcommand{\bfrNoneD}{99.3\%}
\newcommand{\bfrNNA}{23.0\%}\newcommand{\bfrNNB}{71.5\%}
\newcommand{\bfrNNC}{94.2\%}\newcommand{\bfrNND}{99.8\%}
\newcommand{\bfrNFA}{25.2\%}\newcommand{\bfrNFB}{75.0\%}
\newcommand{\bfrNFC}{96.7\%}\newcommand{\bfrNFD}{100.0\%}
\newcommand{\bfrNA}{33.0\%}\newcommand{\bfrNB}{83.7\%}
\newcommand{\bfrNC}{99.3\%}\newcommand{\bfrND}{100.0\%}
\newcommand{\bfrTraj}{600}               % no safety filter here: all loaded
\newcommand{\rdLossless}{14.77}     % log2|S| bits, lossless index

\newcommand{\rdWidthNN}{0.09}
\newcommand{\rdHiNF}{14.52}\newcommand{\rdLoNF}{13.41}
\newcommand{\rdHNF}{14.35}

\newcommand{\rdBitSaveNF}{0.25}    % bits saved at the shipped threshold
\newcommand{\rdBitSavePctNF}{1.69\%} % ... as a share of the lossless rate
\newcommand{\rdTokSaveNF}{22.50\%}    % tokens removed at the same point
\newcommand{\rscBase}{75.83\%}        % observed resolve rate, uncompacted
\newcommand{\rscNN}{75.83\%}\newcommand{\rscNF}{75.67\%}
\newcommand{\rscN}{70.33\%}
\newcommand{\rscMaskFifty}{75.17\%}\newcommand{\rscMaskTwenty}{66.33\%}

\newcommand{\rduTokNN}{18.89\%}\newcommand{\rduTokNF}{22.50\%}
\newcommand{\rduTokN}{32.02\%}
\newcommand{\rscNineEight}{75.83\%}\newcommand{\rscNineSix}{75.83\%}
\newcommand{\rduTokNineEight}{19.64\%}\newcommand{\rduTokNineSix}{21.53\%}
\newcommand{\rduTokMaskFifty}{4.56\%}\newcommand{\rduTokMaskTwenty}{34.28\%}

\newcommand{\rduTraj}{600}\newcommand{\rduResolved}{455}
\newcommand{\rscLossNF}{0.17}      % points of resolve rate at risk, shipped
\newcommand{\rscEighty}{35.67\%}\newcommand{\rduTokEighty}{64.38\%}

\newcommand{\margQ}{1{,}536}       % evidence-bearing questions (category 5 excluded)
\newcommand{\margRec}{5{,}882}     % dialogue-turn records; C(5882,2) = \locomoTotal
\newcommand{\margDim}{768}         % bge-base-en-v1.5
\newcommand{\margMedOne}{0.0185}   % median rank-k / rank-(k+1) margin, k=1
\newcommand{\margMedThree}{0.0056} % ... k=3
\newcommand{\margMedTwenty}{0.0010}% ... k=20
\newcommand{\margBoundNN}{0.1414}  % l2 bound at cos >= 0.99
\newcommand{\margBoundNF}{0.3162}  % ... at cos >= 0.95, the shipped threshold
\newcommand{\margCertNNF}{3.71\%}  % certify at k=1, cos >= 0.995
\newcommand{\margCertNN}{0.46\%}   % ... cos >= 0.99
\newcommand{\margBoundVsMarg}{17$\times$} % 0.3162 / 0.0185          D
\newcommand{\margCertBestEps}{0.001}    % only non-trivial column: linf, eps = 0.001
\newcommand{\margCertBest}{15.62\%}     % ... certifying at k=1
\newcommand{\margCertBestThree}{0.59\%} % ... already collapsed by k=3

\newcommand{\lhtbTraj}{1{,}459}
\newcommand{\lhtbModels}{28}
\newcommand{\lhtbSteps}{131{,}357}
\newcommand{\lhtbStepsMed}{41}
\newcommand{\lhtbStepsPNinety}{230}
\newcommand{\lhtbStepsMax}{1{,}074}
\newcommand{\lhtbPromptTok}{12.1B}

\newcommand{\ordIndepNineFive}{29.35\%}  % independence estimate, rejected
\newcommand{\ordANineFive}{20.52\%}      % A: line-exact -> near-dup
\newcommand{\ordBNineFive}{22.97\%}      % B: near-dup -> line-exact (shipped)
\newcommand{\ordCostNineFive}{2.45}      % pts lost by running the free layer first

\newcommand{\ordCostNinety}{4.08}
\newcommand{\detChunks}{27{,}986}
\newcommand{\detRecords}{15{,}169}
\newcommand{\detBatches}{$\{8,16,32,64,128,256\}$}
\newcommand{\detNBatches}{six}
\newcommand{\detJitterEmb}{$1.1\times10^{-6}$}   % max |delta| per coordinate
\newcommand{\detJitterCos}{$1.8\times10^{-7}$}   % max |delta cos| vs batch 64
\newcommand{\detJaccard}{1.000000}
\newcommand{\detSymDiff}{0}
\newcommand{\detDropNineNine}{2{,}357}
\newcommand{\detDropNineFive}{4{,}439}
\newcommand{\detHw}{NVIDIA GB10}

\newcommand{\hrDecisions}{27{,}985}

\newcommand{\hrMinNineNine}{$6.7\times10^{-6}$}
\newcommand{\hrMinNineFive}{$8.3\times10^{-6}$}
\newcommand{\hrRatioNineNine}{37.7$\times$}
\newcommand{\hrRatioNineFive}{46.6$\times$}
\newcommand{\hrPOneNineNine}{0.0013}
\newcommand{\hrMedNineNine}{0.077}
\newcommand{\hrWithinTenX}{zero}

\newcommand{\bandTraj}{200}
\newcommand{\bandPairs}{283{,}980}
\newcommand{\bandRecords}{5{,}011}
\newcommand{\bandRs}{$\{2,4,8,16,32\}$}
\newcommand{\bandRels}{$0.01$--$0.3$}
\newcommand{\bandBestR}{8}
\newcommand{\bandBestRel}{0.2}
\newcommand{\bandBestCand}{1.46\%}
\newcommand{\bandBestSpeedup}{68$\times$}
\newcommand{\bandBestPrec}{0.565}
\newcommand{\bandShipR}{4}
\newcommand{\bandShipRel}{0.2}
\newcommand{\bandShipCand}{31.6\%}
\newcommand{\bandShipSpeedup}{3.2$\times$}
\newcommand{\bandRelaxR}{8}
\newcommand{\bandRelaxRel}{0.3}
\newcommand{\bandRelaxRecall}{0.979}

\newcommand{\bandRelaxSpeedup}{26$\times$}
\newcommand{\bandWideRecallNineNine}{0.770}
\newcommand{\bandWideRecallNineFive}{0.314}
\newcommand{\bandNarrowCand}{90.9\%}
\newcommand{\bandNarrowRel}{0.08}
\newcommand{\bandSigma}{$0.0361$}

\newcommand{\sdSeeds}{five}
\newcommand{\sdClusters}{559}
\newcommand{\sdRateNineNine}{28.0\%}
\newcommand{\sdRateSpread}{0.14\%}       % spread of drop-set SIZE across seeds
\newcommand{\sdJaccardNineNine}{0.686}
\newcommand{\sdJaccardNineFive}{0.731}
\newcommand{\sdStableNineNine}{43.7\%}   % intersection / union of the drop sets
\newcommand{\sdStableNineFive}{50.9\%}
\newcommand{\sdFlipNineNine}{5{,}798}    % chunks whose fate depends on the seed

\newcommand{\sdEcrNNN}{19.60\%}          % ECR@0.999, at cos 0.9999
\newcommand{\sdEcrNN}{22.84\%}           % ECR@0.990, at cos 0.995  -- beats ours
\newcommand{\sdEcrNF}{33.74\%}           % ECR@0.950, at cos 0.95   -- beats ours
\newcommand{\sdSafeMicroTight}{99.917\%} % best retention reached at any t swept

\newcommand{\xhwHwA}{NVIDIA GB10}
\newcommand{\xhwHwB}{NVIDIA A100-SXM4-40GB}
\newcommand{\xhwTorchA}{2.13.0+cu132}
\newcommand{\xhwTorchB}{2.14.0}

\newcommand{\xhwDropNineNine}{7{,}678}
\newcommand{\xhwDropNineFive}{11{,}727}
\newcommand{\xhwDigestNineNine}{\texttt{b46a58ac82d22bfe}}
\newcommand{\xhwDigestNineFive}{\texttt{8e788921d75c7e76}}
\newcommand{\xhwVerdict}{The same probe run on an \xhwHwB{} under a different
PyTorch build (\xhwTorchB{} against \xhwTorchA{}) produces the same two deletion
sets byte for byte: \xhwDropNineNine{} chunks dropped at $\cos \ge 0.99$ and
\xhwDropNineFive{} at $\cos \ge 0.95$, digests \xhwDigestNineNine{} and
\xhwDigestNineFive{} on both machines.}

\newcommand{\lldModel}{Qwen2.5-7B-Instruct}
\newcommand{\lldTraj}{twelve}
\newcommand{\lldRepeats}{ten}
\newcommand{\lldRecords}{30}
\newcommand{\lldConcurrency}{16}
\newcommand{\lldJaccard}{1.000}
\newcommand{\lldStable}{100\%}
\newcommand{\lldNonEmpty}{four}
\newcommand{\lldEmpty}{eight}
\newcommand{\lldSizes}{$\{8, 10, 15, 29\}$}
\newcommand{\llmdetStatus}{On \lldTraj{} LHTB trajectories of \lldRecords{}
records each, \lldRepeats{} repeats at temperature $0$ with a fixed seed, the
DELETE set was identical every time: pairwise Jaccard \lldJaccard{},
\lldStable{} of trajectories returning a single distinct set. Repeating the
measurement with the same prompt scheduled against \lldConcurrency{} concurrent
decoys, so that batch composition varied between repeats, changed nothing.}

\newcommand{\cmpPoints}{84}
\newcommand{\cmpFront}{five}
\newcommand{\cmpBaseRecall}{0.6536}
\newcommand{\cmpIndexN}{5{,}882}
\newcommand{\cmpShipCos}{0.95}
\newcommand{\cmpShipDedup}{1.188$\times$}
\newcommand{\cmpShipCodec}{3.90$\times$}
\newcommand{\cmpShipComposite}{4.63$\times$}
\newcommand{\cmpShipRecall}{0.6543}

\newcommand{\cmpShipJaccard}{0.990}
\newcommand{\cmpKneeComposite}{9.11$\times$}
\newcommand{\cmpKneeRecall}{0.6491}

\newcommand{\cmpKneeJaccard}{0.903}
\newcommand{\cmpMaxComposite}{42.88$\times$}
\newcommand{\cmpMaxRecall}{0.6204}

\newcommand{\cmpMaxJaccard}{0.560}

\newcommand{\dhTraj}{600}
\newcommand{\dhChunks}{27{,}986}
\newcommand{\dhTokens}{9.07M}

\newcommand{\dhRatioNineFive}{0.559}
\newcommand{\dhRatioNineNine}{0.350}
\newcommand{\dhRatioNineEight}{0.409}
\newcommand{\dhRatioNineSix}{0.532}
\newcommand{\dhHalfNineNine}{63.2\%}

\newcommand{\rkQueries}{13{,}133}
\newcommand{\rkChunks}{29{,}322}
\newcommand{\rkKeepNineFive}{83.7\%}
\newcommand{\rkTenNineNine}{0.9978}
\newcommand{\rkTenNineEight}{0.9947}
\newcommand{\rkTenNineSix}{0.9881}
\newcommand{\rkTenNineFive}{0.9846}
\newcommand{\rkTenNine}{0.9518}
\newcommand{\rkKeepNineNine}{0.916}\newcommand{\rkKeepNineEight}{0.896}
\newcommand{\rkKeepNineSix}{0.856}\newcommand{\rkKeepNinety}{0.696}
\newcommand{\rkOneNineFive}{0.9606}
\newcommand{\rkTwentyNineFive}{0.9928}
\newcommand{\rkKeepNine}{69.6\%}

\newcommand{\ovGpus}{four NVIDIA A100-SXM4-40GB}
\newcommand{\ovCores}{64}
\newcommand{\ovNuma}{four}
\newcommand{\ovModel}{Qwen2.5-7B-Instruct}
\newcommand{\ovTP}{4}
\newcommand{\ovRate}{5\,req/s}

\newcommand{\ovRateBaseTtft}{157.0\,ms}
\newcommand{\ovRateTputEight}{1.0000$\times$}
\newcommand{\ovRateTputSixteen}{0.9999$\times$}

\newcommand{\ovRateTtftSixteen}{166.5\,ms}
\newcommand{\ovRateTtftDeltaSixteen}{9.5\,ms}
\newcommand{\ovRateTtftThirtytwo}{192.5\,ms}
\newcommand{\ovRateTpotSixteen}{1.0027$\times$}
\newcommand{\ovRateTpotThirtytwo}{1.1585$\times$}
\newcommand{\ovSoloSixteen}{5{,}001}
\newcommand{\ovCompactRateSixteen}{4{,}647}

\newcommand{\ovVsSoloSixteen}{0.929}
\newcommand{\ovVsSoloThirtytwo}{0.620}

\newcommand{\ovSatTputSixteen}{0.9990$\times$}

\newcommand{\ovSatTtftSixteen}{1.0004$\times$}
\newcommand{\ovSatVsSoloSixteen}{0.907}
\newcommand{\ovPinTput}{0.9966$\times$}
\newcommand{\ovPinVsSolo}{0.540}
\newcommand{\ovKeepUp}{4.7$\times$}

\newcommand{\lhtbEmbTraj}{400}
\newcommand{\lhtbWindows}{138{,}652}
\newcommand{\lhtbWindowsObs}{64{,}510}

\newcommand{\lhtbTokObs}{22.0M}
\newcommand{\lhtbExactObs}{3.35\%}

\newcommand{\lhtbNearNineNineObs}{16.78\%}
\newcommand{\lhtbNearNineFiveObs}{33.39\%}

\newcommand{\lhtbNearNineFiveAll}{33.07\%}

\newcommand{\lhtbRatioNineNine}{2.6$\times$}
\newcommand{\lhtbRatioNineFive}{2.4$\times$}
\newcommand{\kvCtx}{360K}        % tokens, the re-prefill example in the intro

\newcommand{\floorPoints}{35}            % M: operating points on the micro/macro plane
\newcommand{\floorFamilies}{four}        % ours, masking, SemDeDup, exact hashing
\newcommand{\floorSE}{1.75}              % D: binomial s.e. of 455/600, in points
\newcommand{\floorMacroNNN}{99.3\%}      % M: ... at tau = 0.999  (SemDeDup cos 0.9999)
\newcommand{\floorDmgNNN}{0.7\%}
\newcommand{\floorDmgNN}{5.5\%}

\newcommand{\floorRscNNN}{0.17}          % (ours t = 0.95)
\newcommand{\floorRscNN}{1.83}           % (ours t = 0.92)
\newcommand{\floorCutHold}{99.917\%}     % M: loosest point holding macro >= 99%
\newcommand{\floorCutFail}{99.889\%}     % M: tightest point failing it -- ours, t = 0.93
\newcommand{\floorFailMacro}{98.997\%}   % M: 592/598 at that point
\newcommand{\sdSafeMicroTightish}{99.092\%} % M: SemDeDup at cos 0.995

\newcommand{\relaxDilation}{4$\times$}   % window = 4 x the line's token count
\newcommand{\relaxOwn}{100.0\%}          % M: found in its own trajectory
\newcommand{\relaxCtrl}{1.21\%}          % M: found in an unrelated one, at 4x
\newcommand{\relaxCtrlOne}{0.85\%}       % M: at 1x (contiguous)
\newcommand{\relaxCtrlEight}{1.72\%}     % M: at 8x
\newcommand{\idxMeanMB}{0.128\,MB}       % M: mean state per trajectory
\newcommand{\idxPNineNineMB}{0.34\,MB}   % M: p99
\newcommand{\idxMaxMB}{0.44\,MB}         % M: max over 600
\newcommand{\idxPerKtok}{8.5\,KB}        % M: per 1K observation tokens
\newcommand{\idxShare}{94\%}             % M: embeddings as share of the state
\newcommand{\idxGrowth}{0.71}            % M: exponent of bytes vs tokens, log-log
\newcommand{\idxOverKV}{0.015\%}         % D: state / KV cache of the same tokens
\newcommand{\idxKVPerTraj}{867\,MB}      % D: that KV cache, mean per trajectory

\newcommand{\hwGbBw}{237\,GB/s}
\newcommand{\hwGbFp}{90.3\,TFLOP/s}
\newcommand{\hwGbMu}{0.87}

\newcommand{\hwAmpereBw}{1{,}365\,GB/s}
\newcommand{\hwAmpereFp}{288.5\,TFLOP/s}
\newcommand{\hwAmpereMu}{0.88}

\newcommand{\reprefillVsMeasured}{8.4$\times$}  % 538 / 64
\newcommand{\bwRatioAgg}{23$\times$}            % 4 x 1,365 / 237

\newcommand{\hwBwRatio}{5.7$\times$}    % 1364.7 / 237.4                    D
\newcommand{\hwEncRatio}{1.4$\times$}   % 99935 / 71708                     D
\newcommand{\hwBwOverstate}{14\%}       % 1555 / 1364.7 - 1                 D
\newcommand{\ppRateLo}{27.3\,K\,tok/s}   % apparent rate at 30,720 tok       S
\newcommand{\ppRateHi}{36.9\,K\,tok/s}   % apparent rate at 4,096 tok        S
\newcommand{\ppMaxLen}{30{,}720}         % served max_model_len              S
\newcommand{\ppWarmSpeedup}{12.1$\times$}% cold 1.124 s vs warm 0.093 s      S
\newcommand{\ppWarmS}{0.09\,s}           % warm prefill at 30,720 tok        S
\newcommand{\ppColdS}{1.12\,s}           % cold prefill at 30,720 tok        S
\newcommand{\ppCachedFrac}{99.7\%}       % warm cached_tokens / prompt_tokens S
\newcommand{\ppReprefill}{64\,s}         % fit extrapolated to 360,152 tok   D
\newcommand{\ppReprefillLin}{13\,s}      % same, attention term dropped      D
\newcommand{\ppLinUnderstate}{4.8$\times$}% 64 / 13                          D

\newcommand{\encSensTraj}{300}
\newcommand{\encSensChunks}{14{,}756}
\newcommand{\encSensBge}{13.79\%}        % shipped: bge-base at tau = 0.95    M
\newcommand{\encSensEfive}{30.57\%}      % e5-base at the same tau            M
\newcommand{\encSensMini}{11.26\%}       % MiniLM-L6 at the same tau          M
\newcommand{\encSensSpread}{2.7$\times$} % 30.57 / 11.26                      D
\newcommand{\encSensRecallFixed}{99.6--99.7\%} % e5, gte recall of bge's set  M
\newcommand{\encSensTstarEfive}{0.975}   % tau matching 13.79\% removal       M
\newcommand{\encSensTstarMini}{0.935}    %                                    M
\newcommand{\encSensJacEfive}{0.76}      % Jaccard vs shipped, matched        M
\newcommand{\encSensJacGte}{0.85}        %                                    M
\newcommand{\encSensJacMini}{0.65}       %                                    M
\newcommand{\encSensRecallEfive}{85.9\%} % recall of shipped set, matched     M
\newcommand{\encSensRecallGte}{90.9\%}   %                                    M
\newcommand{\encSensRecallMini}{75.8\%}  %                                    M

\newcommand{\ixSurvMed}{36}              % survivors per trajectory, median   M
\newcommand{\ixSurvMax}{112}             % survivors per trajectory, max      M
\newcommand{\ixAdmitMs}{16.4\,ms}        % one admit() call, median           M
\newcommand{\ixSearchUs}{44\,$\mu$s}     % the search inside it, median       M
\newcommand{\ixSearchFrac}{0.46\%}       % search / total replay wall clock   M
\newcommand{\ixSearchPerK}{13.4\,$\mu$s} % added search per 1,000 survivors   M
\newcommand{\ixEffBw}{229\,GB/s}         % effective, at N = 1M               M
\newcommand{\ixBwOfPeak}{97\%}           % 229 / 237 achieved triad           D
\newcommand{\ixCrossover}{1.2\,M}        % survivors where search = encode    D
\newcommand{\ixCrossoverRatio}{10{,}000$\times$} % 1.22M / 112                D
\newcommand{\ixCrossoverMem}{3.8\,GB}    % 1.22M x 768 x 4 B                  D
\newcommand{\ixReachedMem}{344\,KB}      % 112 x 768 x 4 B                    D

\newcommand{\sdwNineFiveRem}{13.69\%}    % within-trajectory, cos 0.95         M
\newcommand{\sdwNineFiveMicro}{99.939\%} % evidence retained there             M
\newcommand{\sdwNineNineRem}{6.45\%}     % within-trajectory, cos 0.99         M
\newcommand{\sdwReachFactor}{2.46$\times$}% 33.74 / 13.69, the value of reach  D
\newcommand{\sdNineFiveRem}{33.74\%}     % global k-means, cos 0.95            M
\newcommand{\sdNineNineRem}{24.55\%}     % global k-means, cos 0.99            M
\newcommand{\sdwSeedSpread}{0.03\,pts}   % removal range over five seeds       M

%% file: sections/00-abstract.tex
Agent memory grows as agents read inputs, reason, and call tools. Longer histories
increase inference cost and eventually exceed the context window. LLM-based
summarization reduces this history but adds latency and provides no explicit
bound on information loss. We propose LAM, a \textbf{L}ossy \textbf{A}gent
\textbf{M}emory system with three components: a deterministic deduplication rule
with a substitution bound on retrieval scores --- a bound on score
perturbation, not a certificate of unchanged ranking, a memory manager that preserves the cached
prefix and overlaps compaction with inference, and a performance model that
estimates compaction costs before deployment. On \subTraj{} agent trajectories,
LAM removes \safOursRemoval{} of observation tokens while retaining
\safOursMicro{} of the measured gold-patch evidence. At a fixed deletion set,
the performance model predicts a \polGainOne{}--\polGainBs{} end-to-end
speedup from removing records before prefill instead of deleting them from a
prefilled context. That benefit comes from the schedule rather than the rule and applies
to any prefix-preserving test.

%% file: sections/01-intro.tex
\section{Introduction}
\label{sec:intro}

\begin{figure}[!b]
\centering
\includegraphics[width=\linewidth]{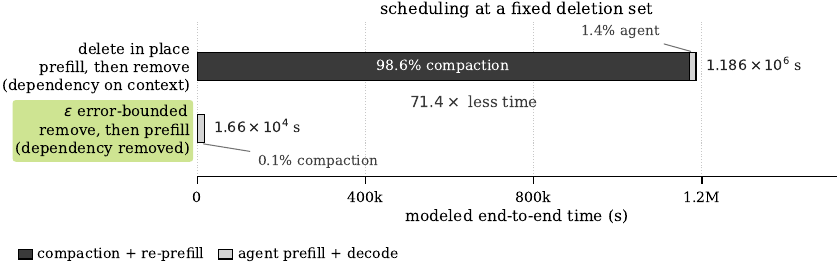}
\caption{\textbf{Compaction needs the right schedule.} Modeled end-to-end time
under the two schedules, where the grey bars delete a record already in the
context and so pay to prefill it and then prefill the invalidated suffix again,
while ours (green) tests the arriving record first and pays neither.}

\label{fig:motivation}
\end{figure}

AI agents retain the files they read, the commands they run, and the resulting
outputs. This history supports task completion
\citep{chhikara2025mem0,yu2026memagent}, but its growth raises inference cost and
latency, and once it exceeds the context window some of it must go. Memory
compaction decides what to retain, rewrite, compress, or discard
\citep{chhikara2025mem0,fang2026lightmem,chen2026slipstream}; during long tasks it
runs repeatedly and can block inference for tens of seconds per event
\citep{cim2026parallel,kummer2026wild}.

Existing methods select information for later reuse. LLM summarization and
learned compressors process the context with a model
\citep{mu2023gisting,ge2024icae,cheng2024xrag,pan2024llmlingua2}, adding inference
cost; SemDeDup \citep{abbas2023semdedup} clusters embeddings with $k$-means, whose
output can vary with its initialization; observation masking keeps recent tool
observations and discards older ones \citep{lindenbauer2025complexitytrap}. Each
exposes a token budget, cluster count, or window size, but none of those
parameters bounds the retrieval-score loss that deletion causes. Furthermore
their reproducibility depends on the model and execution configuration
(\S\ref{sec:determinism}).

Tool outputs often repeat: agents reread files, rerun tests, and list directories
more than once. Hashing \citep{quinlan2002venti,zhu2008datadomain} detects exact
copies but misses outputs differing in small details. LAM therefore keeps exact
hashing as its first layer --- the mechanism serving systems already use to reuse
identical KV blocks \citep{kwon2023pagedattention} --- and adds a second above it
that merges near-duplicates whenever unit-normalized embeddings satisfy
$\|a-b\|_2 \le \delta$, a query-free test equivalent to a merge threshold on
cosine. Thus every deleted record keeps a retained representative within
$\delta$, and for every unit query the maximum retrieval score decreases by at
most $\delta$ (Theorem~\ref{thm:main}). The guarantee concerns scores, not
ranking or task outcomes, which we evaluate separately (\S\ref{sec:eval}). With
fixed embeddings and arrival order the rule is reproducible and reports the
maximum substitution distance it observed, $\hat\delta$.

\paragraph{Compaction scheduling.}
The rule tests each record on arrival, before prefill, so the prefilled prefix
survives. Deleting from an existing context instead shifts subsequent positions
and invalidates their KV entries. Specifically, re-prefilling a \kvCtx{} context
costs ${\approx}\reprefillCost$\,s in our cost model against
${\approx}\turnServeCost$\,s for the turn's inference: over \subTraj{}
trajectories removing the same records either way, deletion in place spends
\polShareDelete{} of the run on compaction and re-prefill against
\polShareAdmit{} for ours (Figure~\ref{fig:motivation}). The test can also run on the host CPU while the
accelerator serves; encoding is charged when embeddings are not already
available.

\paragraph{Evidence retention.}
We score evidence retention against each SWE-smith task's gold patch and count
lines accessible only through reference stubs as lost. At a retention floor of
$0.999$, none of the evaluated observation-masking settings qualifies; SemDeDup
removes \sdEcrNNN{} of tokens, compared with LAM's \ecrNNNOurs{}.
At a $0.99$ floor, SemDeDup removes more (\S\ref{sec:eval}). LAM's removal rate
rises from \growthSmallRem{} at \growthSmallT{} tokens to \growthLargeRem{} at
\growthLargeT{}. On LHTB, which contains \lhtbTraj{} terminal-agent trajectories
from \lhtbModels{} models under a different scaffold, the near-duplicate rule
removes \lhtbNearNineFiveObs{} of observation tokens, compared with
\nearGreedyNineFive{} on SWE-smith. Near-duplicates are rare in the LoCoMo
conversational-memory control.

\paragraph{Contributions.} Figure~\ref{fig:sysdesign} summarizes LAM.

\begin{itemize}\itemsep1pt
\item \textbf{An online error-bounded test} (\S\ref{sec:online}). LAM tests each
record on arrival, embeds it once, and preserves the prefilled prefix, at
\turnLatMean{} per turn (\turnLatPnn{} at p99) and with the same deletion set
as the batch pipeline on \onlineEqBatch{} trajectories.
\item \textbf{A query-independent substitution bound} (\S\ref{sec:bound}).
Replacing each deleted record with a retained one within $\delta$ limits the
retrieval-score deficit to $\delta$ for every unit query. The bound is tight
and set by the merge threshold; each run reports its realized value,
\deltaAttained{} at a configured \deltaConfigured{}.
\item \textbf{A scheduling and placement study} (\S\ref{sec:main-results}). At a
fixed deletion set the cost model predicts \polGainOne{} speedup at batch 1 and
\polGainBs{} at batch 32 for removal before prefill, and separates LAM from
masking, summarization, SemDeDup, and LLMLingua-2 by
\eeMaskOne{}--\eeLinguaOne{} end to end. CPU compaction alongside \ovModel{}
serving on \ovGpus{} maintains \ovRateTputSixteen{} of baseline output
throughput (\S\ref{sec:ablation-study}).
\item \textbf{An evaluation of removal and evidence retention}
(\S\ref{sec:main-results}). ECR@$\tau$ measures removal subject to an explicit
retention floor: at the default threshold LAM removes \safOursRemoval{} of
observation tokens and retains \safOursMicro{} of measured evidence, with
deletion sets identical across the tested batch sizes and hardware.

\end{itemize}

%% file: sections/04-method.tex
\section{Design}
\label{sec:method}

LAM combines bounded deduplication (\S\ref{sec:relation}--\ref{sec:bound}),
a memory manager that overlaps compaction with inference
(\S\ref{sec:sysdesign}), and a model of deployment cost
(\S\ref{sec:perfmodel}).

\subsection{The proximity relation}
\label{sec:relation}

Let records be embedded as unit-normalized vectors in $\mathbb{R}^D$. Two
records $a, b$ are \emph{$\varepsilon$-equivalent} when $\|a - b\|_2 \le \delta$.
This relation is not transitive, and for unit vectors it is equivalent to a
threshold on cosine:

\begin{equation}
\|a-b\|_2^2 = 2 - 2\cos(a,b),
\qquad\text{so merging at }\cos \ge t \iff \|a-b\|_2 \le \delta = \sqrt{2-2t}.
\label{eq:identity}
\end{equation}
We call $t$ the \emph{merge threshold} and $\delta$ the \emph{substitution
bound}; $\varepsilon$ denotes the error budget generically and $\delta$ its
value under $\ell_2$. Each fixes the other, and \S\ref{sec:bound} shows the same
$\delta$ bounds the retrieval-score loss: the radius of the query-free test is
the size of the guarantee. We use $\ell_2$ because a coordinatewise tolerance
$\|a-b\|_\infty \le 2\varepsilon$ would yield a bound proportional to $\sqrt{D}$
instead. The quantization step in $Q_\varepsilon$ only proposes candidates and
does not enter the bound (Appendix~\ref{app:margin}).

\subsection{Three layers}
\label{sec:layers}

\textbf{L0, byte-exact hashing.} This layer is carried over from LLM serving:
vLLM hashes fixed-size blocks so a repeated prefix reuses one copy of the KV
cache instead of recomputing it \citep{kwon2023pagedattention}. We apply the
same mechanism one level up, to the text an agent accumulates rather than the
cache it produces --- xxHash3 over individual lines, repeats replaced by stubs
--- lines being the granularity that removes the most under exact matching
(Appendix~\ref{sec:redundancy}). L0 thus inherits the serving cache's need for
an exact match, which L1 removes, and runs after L1 to preserve the blocks
near-duplicate matching uses (Figure~\ref{fig:joint}).
\textbf{L1, $\varepsilon$-equivalence candidates.} We quantize each coordinate
with $Q_\varepsilon(x) = \mathrm{round}(x / 2\varepsilon)$, divide the vector
into bands, and hash each band; records sharing a band digest are candidates and
every candidate pair is checked in full-precision $\ell_2$. Hashing is $O(N)$ at
fixed dimension and band settings, while verification cost scales with the
candidate count. Missed candidates change removal but never invalidate the
substitution bound on accepted merges.
\textbf{L2, error-bounded byte compression.} We compress retained embeddings
with cuSZp \citep{huang2023cuszp} under a relative \emph{codec error bound},
separate from the substitution bound. Deduplication and
byte compression multiply their storage reductions. Byte compression does not
reduce the text in the context (\S\ref{sec:composite}).

\subsection{The $\varepsilon$ error-bounded test}
\label{sec:online}

LAM applies the redundancy test before each arriving record is prefilled. The
test is the bound: an arrival is dropped only when some retained
record lies within $\ell_2$ distance $\delta = \sqrt{2-2t}$ of it, the radius that
\eqref{eq:identity} assigns to the merge threshold $t$. Every merge is
therefore a $\delta$-substitution (Definition~\ref{def:sub}), and the error it
introduces is at most $\delta$ for every query rather than on average
(Theorem~\ref{thm:main}); at $\cos \ge 0.95$, $\delta = \deltaConfigured{}$.
The operator also reports the largest distance it used, $\hat\delta \le \delta$.
Algorithm~\ref{alg:admit} describes the rule.

\begin{algorithm}[t]
\caption{The $\varepsilon$ error-bounded test for one turn, in which distance
checks use full-precision embeddings while the byte codec compresses only the
embeddings that are retained.}
\label{alg:admit}
\begin{algorithmic}[1]
\Require arrivals $x_1{\dots}x_n$; retained set $R$ and band index $I$; threshold
$t$; step $\varepsilon$
\Ensure updated $R, I$; stubs; measured bound $\hat\delta$
\State $\delta \gets \sqrt{2-2t}$;\quad $\hat\delta \gets 0$
\For{$j = 1$ \textbf{to} $n$}                        \Comment{arrival order}
  \State $e_j \gets \textsc{Encode}(x_j)$            \Comment{encode each arrival once}
  \State $B \gets \textsc{BandHash}(Q_\varepsilon(e_j))$
  \State $\mathcal{C} \gets \bigcup_{b \in B} I[b]$  \Comment{candidate records}
  \State $y^{*} \gets \min\{\, y \in \mathcal{C} : \|e_j - e_y\|_2 \le \delta \,\}$ \Comment{verify distance; choose lowest index}
  \If{$y^{*}$ exists}
    \State emit stub $\langle x_j \equiv y^{*}\rangle$; \quad
           $\hat\delta \gets \max(\hat\delta,\ \|e_j - e_{y^{*}}\|_2)$
  \Else
    \State $R \gets R \cup \{x_j\}$;\quad $I[b] \gets I[b] \cup \{j\}\ \forall b \in B$;\quad
           store $\textsc{cuSZp}(e_j)$
  \EndIf
\EndFor
\State \Return $R,\ I,\ \hat\delta$
\end{algorithmic}
\end{algorithm}

The index contains only retained records, and decisions are never revised.
The rule has three properties.
\textbf{Prefix stability:} replacing an arriving duplicate with a stub leaves
all prefilled tokens in place. Later eviction can invalidate the prefix and is
accounted for separately (Appendix~\ref{sec:policy}).
\textbf{Single encoding:} each chunk is embedded once. Encoding, which accounts
for \encodeShare{} of compaction cost, processes only the arriving tokens
$\Delta$ rather than the full context $C$.
\textbf{Measured bound:} the maximum substitution distance $\hat\delta$ is updated
with each merge and reported for the turn; the maximum across turns gives the
run's bound.
Each dropped chunk is replaced by a reference stub identifying its retained
representative. The stub supports later lookup and inspection. We report both
gross removal and net removal after the \stubCost{} input-token cost of stubs.

\subsection{The substitution bound}
\label{sec:bound}

Each merge deletes a record and retains a nearby representative. This can
change retrieval scores and rankings (Appendix~\ref{sec:theory}). The following
bound limits the decrease in the maximum score for any unit query.

\begin{definition}[$\delta$-substitution]
\label{def:sub}
$(R,\pi)$ is a \emph{$\delta$-substitution} of $S$ if $R \subseteq S$,
$\pi : S \to R$ with $\pi|_R = \mathrm{id}$, and
$\|x - \pi(x)\|_2 \le \delta$ for every $x \in S$.
\end{definition}

\begin{theorem}[Score deficit]
\label{thm:main}
Let $(R,\pi)$ be a $\delta$-substitution of $S$. Then for \emph{every} unit
query $q$,
$0 \le \max_{x \in S} q\cdot x - \max_{z \in R} q \cdot z \le \delta$,
and both constants are tight.
\end{theorem}

The proof uses Cauchy--Schwarz (Appendix~\ref{sec:theory}). Both endpoints are
attainable. The bound holds for all unit queries in the same embedding space
without assumptions about their score distribution. It also implies a distribution-free
rate--distortion bound; Appendix~\ref{sec:theory} relates admissible removal to
the corpus's covering number.

\paragraph{Scope of the bound.}
The theorem does not guarantee unchanged rankings or task outcomes. The
rank-invariance test in Appendix~\ref{app:margin} requires a sufficient score
margin. On LoCoMo, the bound $\delta = \margBoundNF{}$ exceeds the median top-1
margin of \margMedOne{} by \margBoundVsMarg{}; even at $\cos \ge 0.99$, only
\margCertNN{} of queries pass the test. At the shipped threshold the bound
therefore certifies score perturbation, not rank: it is not a retrieval
guarantee, and we do not claim one. We therefore measure retrieval quality and evidence retention
separately (\S\ref{sec:eval}).

\paragraph{The substitution map.}
For each deleted $x$, $\pi(x)$ is the lowest-indexed retained candidate within
$\delta$. With fixed embeddings and arrival order, this choice is independent
of candidate traversal order. Each substitution therefore satisfies
Theorem~\ref{thm:main}. Direct comparison with a retained representative is
necessary because the distance relation is not transitive. At $\cos \ge 0.95$,
this rule removes \nearGreedyNineFive{} of observation tokens.

\paragraph{The measured bound.}
Each run reports its maximum substitution distance alongside the configured
bound: \deltaAttained{} versus \deltaConfigured{} at $\cos \ge 0.95$. The
smaller value bounds the substitutions made in that run. Fixed embeddings give
deterministic decisions under the tested arithmetic implementation. Across
machines, encoder outputs vary slightly: measured cosine differences are
\cosRepro{} and the spread in $\hat\delta$ is \deltaSpread{}. No tested pair
crossed the threshold (Appendix~\ref{sec:determinism}); this is an empirical
result, not a guarantee for records arbitrarily close to the threshold.

\subsection{System design}
\label{sec:sysdesign}

\begin{figure}[t]
\centering
\includegraphics[width=\linewidth]{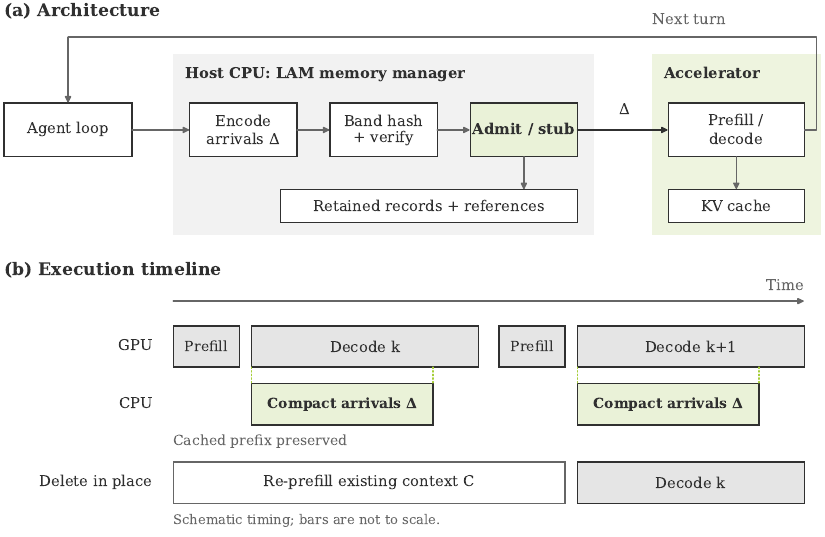}
\caption{System design, where \textbf{(a)} traces one arriving record through
the host memory manager to either a retained record or a reference stub, and
\textbf{(b)} contrasts the two schedules, the upper one voiding the shaded cache
suffix that must then be prefilled again and the lower one leaving the
prefilled prefix intact.}

\label{fig:sysdesign}
\end{figure}

The system uses three design choices (Figure~\ref{fig:sysdesign}), the first
panel of which follows an arrival through candidate generation by band hashing
and exact distance verification.
\textbf{Placement.} The CPU test runs concurrently with GPU inference.
Appendix~\ref{sec:overlap} measures this configuration and finds little serving
interference at the selected thread count. Encoding and byte-codec costs are
accounted for separately when required.
\textbf{Locality.} The band index stores retained records by digest, so each
arrival checks its candidates rather than scanning the full context. The byte
codec reduces embedding storage; text remains available for reference lookup.
\textbf{Granularity.} Encoding processes only the turn's new chunks. Search
cost still depends on the retained index size, which we measure in
Appendix~\ref{sec:index-scaling}.

\subsection{Performance model}
\label{sec:perfmodel}

We estimate deployment cost by dividing a turn into the
part that scales with context length $C$ (KV reads during decode, prefill of
arrivals against the existing context) and the part that does not (streaming
weights once per decode step). Let $\varphi$ be the share of turn latency in the
first part. Eliminating that part gives an Amdahl ceiling of $1/(1-\varphi)$
for speedup from context reduction. Without overlap, modeled end-to-end time is

\begin{equation}
t_{\text{e2e}} \;=\;
\underbrace{\sum_{\text{turns}}\Big[\tfrac{\Delta}{R_{p}} + G_a\, t_{d}(C)\Big]}_{\text{serving}}
+
\underbrace{\sum_{\text{events}}\Big[\Delta\big(\tfrac{1}{R_{e}}+\tfrac{1}{R_{h}}+\tfrac{1}{R_{z}}\big) + \tfrac{|\mathcal{C}|}{R_{v}}\Big]}_{\text{compaction}}
+
\underbrace{\mathbb{1}[\text{not prefix-stable}]\sum_{\text{events}} \tfrac{C}{R_{p}}}_{\text{re-prefill}} .
\label{eq:perf}
\end{equation}

\begin{table}[t]
\centering
\caption{Inputs to \eqref{eq:perf}, split into offline terms measured once per
device by kernel microbenchmark and online terms read from the run itself.}
\label{tab:perfmodel}
\small
\begin{tabular}{llll}
\toprule
 & symbol & quantity & value / source \\
\midrule
\multirow{6}{*}{\rotatebox{90}{\footnotesize offline}}
 & $R_{p}, t_{d}$ & prefill rate, decode step time & GB10 anchors, Table~\ref{tab:anchors} \\
 & $R_{e}$ & encoder throughput & \encRate\,tok/s \\
 & $R_{h}$ & xxHash3 & \hashRate \\
 & $R_{v}$ & exact $\ell_2$ verify & \lTwoRate\,M pairs/s \\
 & $R_{z}$ & cuSZp & \cuszpRate \\
 & $\beta$ & bytes per token & \bytesPerTok\,B/tok \\
\midrule
\multirow{4}{*}{\rotatebox{90}{\footnotesize online}}
 & $C$, $\Delta$, $G_a$ & context, arrivals, decoded tokens & \turnDelta{}, \turnGa{} per turn \\
 & $|\mathcal{C}|$ & candidate pairs verified & run \\
 & $\varphi$ & context-scaling share of a turn & \phiOne{} (b1, 32K), \phiThirtyTwo{} (b32) \\
 & $\mathbb{1}[\cdot]$ & requires re-prefill & $0$ test first, $1$ delete-in-place \\
\bottomrule
\end{tabular}
\end{table}

Table~\ref{tab:perfmodel} separates device measurements
(Appendix~\ref{app:anchors}) from run-dependent inputs, and its indicator is
what decides whether a schedule pays re-prefill at all. The context-dependent
share $\varphi$ varies with batch size, context
length, and model architecture. At batch 1 and $C = 32$K,
$\varphi = \phiOne{}$, giving a context-reduction speedup ceiling of \ceilOne{};
at batch 32 the ceiling is \ceilThirtyTwo{}. These deployment parameters must
therefore accompany speedup comparisons. A prefix-preserving test removes
the re-prefill term. In the model, deletion from a $C = 360$K context adds
$\reprefillCost{}$\,s per event to a turn with \turnServeCost{}\,s of inference.
The overlap experiment in Appendix~\ref{sec:overlap} measures how much of the
compaction term can run concurrently with serving.

\paragraph{Candidate generation.}
We use standard banded hashing and evaluate its operating points on agent
observations (\S\ref{sec:ablation-study}; Appendix~\ref{sec:banding}). Full-precision
verification preserves the substitution bound at every setting. We report
comparison savings at full pair recall so the candidate filter matches
exhaustive greedy deduplication. Lower recall can retain additional records;
it preserves the substitution bound but need not preserve the packing property
behind the removal ceiling.

%% file: sections/08-eval.tex
\section{Experiments}
\label{sec:eval}

\subsection{Experimental setup: benchmarks and baselines}
\label{sec:setup}

\paragraph{Benchmarks.}
The primary corpus contains \corpusTraj{} SWE-smith agent trajectories
(\corpusRecords{} records, \corpusTokens{} tokens); unless stated otherwise we
use a fixed \subTraj{}-trajectory subset with \subObs{} observation records,
\subChunks{} chunks, and \subTokens{} observation tokens, encoding $512$-token
windows with \texttt{bge-base-en-v1.5} ($D = \embedDim{}$). The determinism study uses \detChunks{} of those windows.
LHTB adds \lhtbTraj{} trajectories from \lhtbModels{} models and \lhtbSteps{}
steps, with roughly an order of magnitude more steps per trajectory
(Appendix~\ref{sec:lhtb}); LoCoMo provides a conversational-memory control, in
which we find only \locomoPairs{} near-duplicate pairs among \locomoTotal{}.
Main measurements use an NVIDIA GB10, serving experiments two nodes with
\ovGpus{} on the serving node.

\paragraph{Baselines.}
We report LAM in three cumulative configurations: L0 alone, L0$+$L1, and the
full pipeline. L0, the exact-hashing layer inherited from serving-side prefix
caching (\S\ref{sec:layers}), is therefore an ablation of our own design rather
than a competing system. Evidence retention is measured by searching for gold-patch lines in the residual
text, which applies to deletion methods: observation masking, keeping the last
$k$ observations \citep{lindenbauer2025complexitytrap}, and SemDeDup
\citep{abbas2023semdedup}. A relaxed variant also covers token dropping
(\S\ref{sec:relaxed}), admitting LLMLingua-2 \citep{pan2024llmlingua2}, the
intake compressor of LightMem~\citep{fang2026lightmem} and thus that component
rather than the full system. Neither variant assesses paraphrases or latent
representations, so MemAgent~\citep{yu2026memagent}, Gisting
\citep{mu2023gisting}, ICAE, and xRAG are compared on modeled cost and a
Mem0-style deletion prompt separately. Removal ratios are token-weighted.
Finally, LAM deduplicates within each trajectory while the main SemDeDup
configuration clusters across them (Appendix~\ref{sec:sdwithin}); we compare at
matched removal where applicable.

\subsection{Main results}
\label{sec:main-results}

\begin{figure}[t]
\centering
\includegraphics[width=0.50\linewidth]{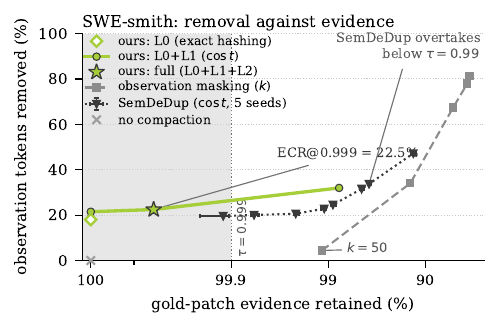}
\caption{Token removal against evidence retention, one marker per measured
setting and up-and-left better, where the three accented series are LAM's
cumulative layers --- L0, L0$+$L1, and the full pipeline --- and the error bars
give SemDeDup's spread across seeds.}

\label{fig:ecr}
\end{figure}

\begin{table}[t]
\centering
\caption{Modeled end-to-end latency per turn, with the ratio to LAM in
parentheses, where ``wall'' defers compaction to the context budget and
``admit'' removes before prefill, the shaded row is the shipped configuration,
and accented cells are the best in their column.}
\label{tab:e2emethods}
\footnotesize
\begin{tabular}{llrrr}
\toprule
method & schedule & $r_{\mathit{eff}}$ & batch 1 & batch 32 \\
\midrule
observation masking ($k=5$)         & wall  & \eeReffMask   & \eeSecMask~(\eeMaskOne)     & \eeSecMaskBs~(\eeMaskBs)     \\
LLM summarization ($\rho=10$)       & wall  & \eeReffSumm   & \eeSecSumm~(\eeSummOne)     & \eeSecSummBs~(\eeSummBs)     \\
SemDeDup (global $k$-means, $0.95$) & wall  & \eeReffSd     & \eeSecSd~(\eeSdOne)         & \eeSecSdBs~(\eeSdBs)         \\
LLMLingua-2 (rate $0.80$)           & wall  & \eeReffLingua & \eeSecLingua~(\eeLinguaOne) & \eeSecLinguaBs~(\eeLinguaBs) \\
\midrule
LAM, L0 only                        & admit & \eeReffLzero  & \eeSecLzero~(\eeLzeroOne)   & \eeSecLzeroBs~(\eeLzeroBs)   \\
\shiprow LAM, full                  & admit & \eeReffOurs   & \best{\eeSecOurs}           & \best{\eeSecOursBs}          \\
\bottomrule
\end{tabular}
\end{table}

\paragraph{Removal and evidence.}
$\mathrm{ECR}@\tau$ is the maximum token removal a method achieves while
retaining at least a fraction $\tau$ of the \safLines{} gold-patch evidence
lines across \safTraj{} trajectories; lines reachable only through stubs count
as lost, SemDeDup is swept finely near perfect retention over \sdSeeds{} seeds,
and L2 compresses stored embeddings rather than context tokens, so it moves
neither axis. At the default threshold LAM removes \safOursRemoval{} of observation
tokens and retains \safOursMicro{} of evidence, losing \safOursLost{} line, and
with perfect retention it removes \ecrOneOurs{}. At $\tau = 0.999$ it removes
\ecrNNNOurs{} against \sdEcrNNN{} for SemDeDup and no qualifying masking setting
(Figure~\ref{fig:ecr}, Table~\ref{tab:ecr}); SemDeDup removes more at lower
floors, \sdEcrNN{} versus \ecrNNOurs{} at $0.99$ and \sdEcrNF{} versus
\ecrNFOurs{} at $0.95$.

We selected the $0.999$ floor after examining \floorPoints{} settings from
\floorFamilies{} families (Appendix~\ref{sec:balance}): there every admissible
setting retains all evidence in at least \floorMacroNNN{} of trajectories, and
among families with a measured conditional resolve-rate ceiling the largest
decrease is \floorRscNNN{} points, about a tenth of a standard error, against
\floorDmgNN{} of tasks losing evidence and \floorRscNN{} points at $0.99$. These
are static evidence metrics, not task-success measurements. SemDeDup's
cross-trajectory search contributes to its higher removal: restricting it to
each trajectory drops removal from \sdNineFiveRem{} to \sdwNineFiveRem{} at
$\cos t = 0.95$. Masking removes \ecrNNMask{} at $\tau = 0.99$, \ecrNNRatio{}
less than LAM; within LAM, L0 alone removes \lineAlone{} with no measured
evidence loss and adding L1 raises removal to \safOursRemoval{}.

\paragraph{Scheduling dominates.}
With the compressor, corpus, and deletion set fixed, the cost model predicts
\polGainOne{} speedup at batch 1 and \polGainBs{} at batch 32 from removing
records before prefill instead of deleting them from a prefilled context
(Figure~\ref{fig:motivation}; Appendix~\ref{sec:policy}). The benefit comes from
the schedule, not the rule: any prefix-preserving test receives it, L0 included.
Table~\ref{tab:e2emethods} puts the published methods on that axis over
\eeTurns{} turns of Llama-3.1-8B on GB10 in the growth regime, granting each
its best measured removal --- all but L0 remove more than LAM --- and the
cheapest schedule its rule permits, deferring to the budget wall rather than
deleting eagerly; uncompacted, the run exhausts the KV cache at turn
\eeNoneTurns{}. They separate by \eeMaskOne{}--\eeLinguaOne{} at batch 1, an
ordering that does not track removal --- LLMLingua-2 removes least and pays
most --- because none of the four is prefix-stable, so each pays a re-prefill
that LAM's schedule never incurs.

\subsection{Ablation study}
\label{sec:ablation-study}

\begin{table}[t]
\centering
\caption{Merge-threshold sweep, in which $\delta = \sqrt{2-2t}$ is what the
theorem permits against the $\hat\delta$ the run attained, ``evidence'' is
gold-patch retention and ``resolve'' the resolve-rate ceiling, and the shaded
row is the shipped default.}
\label{tab:sweep}
\begin{tabular}{crrrrrrr}
\toprule
$\cos t$ & $\delta$ & mean $\hat\delta/\delta$ & keep & R@10 & tokens removed & evidence & resolve \\
\midrule
$0.99$ & 0.1414 & \dhRatioNineNine & \rkKeepNineNine & \rkTenNineNine & \rduTokNN & \safOursNineNineMicro & \rscNN \\
$0.98$ & 0.2000 & \dhRatioNineEight & \rkKeepNineEight & \rkTenNineEight & \rduTokNineEight & \safOursNineEightMicro & \rscNineEight \\
$0.96$ & 0.2828 & \dhRatioNineSix & \rkKeepNineSix & \rkTenNineSix & \rduTokNineSix & \safOursNineSixMicro & \rscNineSix \\
\shiprow $0.95$ & 0.3162 & \dhRatioNineFive & 0.837 & \rkTenNineFive & \rduTokNF & \safOursMicro & \rscNF \\
$0.90$ & 0.4472 & 0.706 & \rkKeepNinety & \rkTenNine & \rduTokN & \safOursNinetyMicro & \rscN \\
\bottomrule
\end{tabular}
\end{table}

\paragraph{Algorithm ablation.}
Table~\ref{tab:sweep} sweeps five merge thresholds over \dhTraj{} trajectories
against an uncompacted resolve rate of \rscBase{}: no substitution exceeds the
configured bound, mean realized displacement being \dhRatioNineNine{} of it at
$\cos \ge 0.99$ and \dhRatioNineFive{} at $0.95$, where recall@10, evidence
retention, and the resolve-rate ceiling stand at \rkTenNineFive{},
\safOursMicro{}, and \rscNF{} against \rkTenNine{}, \safOursNinetyMicro{}, and
\rscN{} at $0.90$; we therefore ship $\cos \ge 0.95$. Layer order matters there
because the layers overlap: assuming independent removal overstates their joint
rate by \indepOverstate{} points (Table~\ref{tab:layers}), and running L0 first
costs \ordCostNineFive{} points at $\cos \ge 0.95$ and \ordCostNinety{} at
$0.90$, since line deletion fragments the blocks L1 would match. Thus we run L1
first, leave the base layer last, and configure L2 for the bound rather than
for peak ratio (Appendix~\ref{sec:composite}). Candidate generation keeps the
test cheap at full pair recall, examining \bandShipCand{} of \bandPairs{}
within-trajectory pairs against \bandBestCand{} at $0.99$, a
\bandShipSpeedup{} reduction; since a band match needs agreement in all $r$
coordinates, looser thresholds need smaller bands, and $r = 32$ reaches only
\bandWideRecallNineFive{} recall at $0.95$.
Furthermore the yield grows with the horizon: scoring at each input size $T$
only the trajectories reaching it, removal rises from \growthSmallRem{} at
$T = \growthSmallT{}$ to \growthLargeRem{} at $T = \growthLargeT{}$ against a
whole-trajectory mean of \growthWholeRem{}, so short contexts understate it. The rule is insensitive to arrival order --- online and batch
residuals are byte-identical on \onlineEqBatch{} trajectories --- but not to
the encoder, four of which differ in removal by \encSensSpread{}
(\encSensMini{} to \encSensEfive{}, against \encSensBge{} by default) and
recover \encSensRecallMini{}--\encSensRecallGte{} of its deletion set at
matched removal, so thresholds need per-encoder calibration
(Appendix~\ref{sec:encsens}). Search is not the bottleneck here:
trajectories retain at most \ixSurvMax{} chunks and search is \ixSearchFrac{}
of test time, matching encoding cost only at \ixCrossover{} synthetic
survivors (Appendix~\ref{sec:index-scaling}).

\paragraph{System performance breakdown.}
The cost model is anchored on measured rather than advertised device numbers:
\hwGbBw{} and \hwGbFp{} on GB10 and \hwAmpereBw{} and \hwAmpereFp{} on an
A100-SXM4-40GB, or \hwGbMu{} and \hwAmpereMu{} of specification, so the
advertised A100 bandwidth would overstate it by \hwBwOverstate{}. Bandwidth
differs by \hwBwRatio{} between the devices but encoder throughput only by
\hwEncRatio{}, so encoder cost does not scale with memory bandwidth.
Furthermore the re-prefill term the schedule avoids also holds on a server:
serving \ovModel{} under vLLM on \ovGpus{} with prefix caching, cold prefill of
\ppMaxLen{} tokens takes \ppColdS{} against \ppWarmS{} once \ppCachedFrac{} of
the prompt is cached, a \ppWarmSpeedup{} difference, and the effective rate
peaks at \ppRateHi{} near 4K tokens before falling to \ppRateLo{}, from which
quadratic and linear fits extrapolate to \ppReprefill{} and \ppReprefillLin{}
at 360K, beyond the measured range.

Placing the compactor on that serving node costs the server almost nothing:
with CPU compaction co-located and the load generator on a second node, output
throughput stays at \ovRateTputSixteen{} of baseline, median time to first
token moves by \ovRateTtftDeltaSixteen{}, and the compactor keeps up at
\ovKeepUp{} the GPU's output-token rate. The node has \ovCores{} cores, so past
16 compactor threads both sides lose.

That contention invites isolating the compactor, but non-uniform memory access
(NUMA) placement does not deliver it: confining compaction to the two of
\ovNuma{} NUMA domains furthest from the GPUs holds server throughput at
\ovPinTput{} but drops compactor throughput to \ovPinVsSolo{} of standalone,
costing it more bandwidth than it saves the server. As a result we ship
unrestricted placement (Appendix~\ref{sec:overlap}).

%% file: sections/02-related.tex
\section{Related work}
\label{sec:related}

\paragraph{Model-based compaction.}
Gisting learns soft tokens representing a prefix~\citep{mu2023gisting}, ICAE
encodes context into memory slots~\citep{ge2024icae}, xRAG projects a document
embedding into a token~\citep{cheng2024xrag}, and LLMLingua-2 selects text with a
token-level classifier~\citep{pan2024llmlingua2}. All require model inference, none
bounds substitution error, and reproducibility depends on weights and execution
settings, whereas our rule records explicit deletions and representatives.

\paragraph{Model-free deduplication.}
Min-hash~\citep{broder1997resemblance}, LSH~\citep{indyk1998approximate}, and
SimHash~\citep{charikar2002similarity} underlie the banded hashing LAM uses for
candidate generation. SemDeDup clusters embeddings offline to
remove near-duplicate training examples~\citep{abbas2023semdedup}, whereas LAM
substitutes under a bound as records arrive: across \sdSeeds{} seeds only
\sdStableNineFive{} of SemDeDup's deleted chunks are common to all runs, while
LAM's are identical (Appendix~\ref{sec:sdwithin}). Error-bounded HPC
compression~\citep{huang2023cuszp} motivates such distortion control.

\paragraph{Agent-memory systems.}
Mem0 issues LLM \textsc{add}/\textsc{update}/\textsc{delete}
operations~\citep{chhikara2025mem0}, observation masking retains the last $k$
observations~\citep{lindenbauer2025complexitytrap}, LightMem filters intake and
consolidates offline~\citep{fang2026lightmem}, MemAgent learns a rewriting
policy over fixed-length memory~\citep{yu2026memagent}, and
\citet{omri2026agentmemory} characterize memory growth in stateful workloads
--- all without a substitution bound.

\paragraph{Compaction off the critical path.}
Slipstream validates candidate compactions in the
background~\citep{chen2026slipstream} and Parallel Context Compaction overlaps
compaction with serving~\citep{cim2026parallel}, both by LLM summarization.
Rewriting a prefix, however, forces re-prefill of the affected entries and
competes for the GPU, whereas LAM removes before prefill, so the prefix
survives and the test runs on the CPU.

\paragraph{Serving systems and prefix reuse.}
PagedAttention supports block-level KV reuse~\citep{kwon2023pagedattention},
CachedAttention preserves caches across turns~\citep{gao2024cachedattention},
Preble schedules for prefix sharing~\citep{srivatsa2024preble}, and SMetric
balances agent sessions~\citep{wang2026smetric}; LAM complements them by
preserving prefixes during compaction, and adopts their exact-match hashing as
its first layer.

\paragraph{Performance models and rate--distortion.}
Vidur simulates inference latency~\citep{agrawal2024vidur} and KVServe models KV
compression under service constraints~\citep{liu2026kvserve}; ours adds the
error threshold, compaction cost, and re-prefill cost, while Kolmogorov
$\epsilon$-entropy~\citep{kolmogorov1959epsilon,posner1971epsilon} gives the
covering and packing framework for our removal bounds. Prior work studies
prompt compression in expectation~\citep{nagle2024limits},
decision-preserving~\citep{zou2026decision} and KV-cache
rate--distortion~\citep{colaco2026ratedistortion}; our worst-case bound implies
one for every source on the record set (Corollary~\ref{cor:rd}), complementing
evaluations of task success~\citep{zhang2025empirical,kummer2026wild},
cost~\citep{johnson2026production}, and retrieval~\citep{chai2026kvmem}.

\paragraph{Scope of comparison.}
LAM's scheduling benefit depends on the cost of rebuilding invalidated cache
entries and its removal benefit on observation-stream near-duplicates, so cheap
cache repair, low redundancy, or poor retention would limit it; we evaluate
each separately.

%% file: sections/11-conclusion.tex
\section{Conclusion}
\label{sec:conclusion}

LAM removes redundant agent observations with a query-independent distance test
and bounds the retrieval-score deficit. Because removal precedes prefill the
prefix survives, which the cost model prices at \polGainOne{}--\polGainBs{}
over deletion from an existing context --- a benefit any prefix-preserving rule
receives, LAM's exact-hashing floor included. Above it the near-duplicate layer
costs \costRatioHash{} more per event, yet the two remove more together at
\safOursMicro{} retention, every deletion auditable.

%% file: sections/10-statements.tex
\subsection*{AI use statement}

% AUTHORS: verify this declaration against your actual use before submission.
% One paragraph per box ticked on the submission form, in the form's order.

\paragraph{Aid or polish writing.}
We used generative AI tools to polish the manuscript: tightening prose,
improving readability, and checking notation and cross-references for
consistency. The authors approved the final wording and are responsible for
every claim it makes.

\paragraph{Retrieval and discovery.}
We used the same tools to develop search utilities, to crawl the literature and
public code repositories for related systems, and to summarize the candidates
returned. They located work; they did not judge it. The authors read each cited
paper in the original and checked every claim attributed to it against the
primary source.

\paragraph{Execution.}
The tools assisted with coding, debugging, and code review for the measurement
harness, the baselines, and the plotting code. The authors specified what each
program had to compute and verified its output against the recorded results.
The research question, methodology, experimental design, analyses, and
conclusions are the authors', and so are the substitution bound and its corollaries.
No reported measurement was generated, imputed, or evaluated by an AI tool.

\subsection*{Ethics statement}

This work studies agent execution traces from public research datasets. It
involves no human subjects or personal data. Deduplication can remove
information an agent later needs. Section~\ref{sec:eval} measures evidence loss;
live task success remains unmeasured.

\subsection*{Reproducibility statement}

With fixed embeddings, arrival order, and arithmetic implementation, the
rule produces deterministic deletion indices and representatives.
Choosing the lowest-indexed eligible representative removes dependence on
candidate traversal order. Appendix~\ref{sec:determinism} tests byte-codec
repeatability and deletion-set stability across encoder batch sizes and hardware.
Encoder outputs are not bit-exact across configurations, but no tested decision
crossed the merge threshold. Hardware constants and benchmark sizes are in
Appendix~\ref{app:anchors}; proofs are in Appendix~\ref{sec:theory} and
Appendix~\ref{app:proofs}. Anonymized code and the measurement harness are
provided as supplementary material for review and will be released on acceptance.

%% file: sections/impl.tex
\section{Implementation details}
\label{sec:impl}

Table~\ref{tab:config} lists the default configuration.
Algorithm~\ref{alg:batch} applies the rule to a complete trajectory.
It processes records in the same order as Algorithm~\ref{alg:admit} and produces
identical residuals on \onlineEqBatch{} tested trajectories
(Appendix~\ref{app:eval}).

\begin{table}[h]
\centering
\caption{Default configuration. Parameter sweeps are in
Appendix~\ref{sec:banding} (band width, quantizer step),
Appendix~\ref{app:eval} (threshold) and Table~\ref{tab:anchors} (device).}
\label{tab:config}
\footnotesize
\begin{tabular}{ll}
\toprule
stage & setting \\
\midrule
chunking & fixed 512-token windows over tool observations, no overlap \\
encoder & \texttt{bge-base-en-v1.5}, unit-normalized, $D = 768$ \\
threshold & $\cos t = 0.95$, i.e.\ $\delta = \deltaConfigured{}$ by \eqref{eq:identity} \\
candidates & banded LSH over $Q_\varepsilon$, verified by exact $\ell_2$ \\
byte codec & cuSZp, relative error bound $1\%$ of $\max|x|$ \\
stubs & one line per drop, charged as input (\stubCost{}) \\
devices & \xhwHwA{} and \xhwHwB{} (Table~\ref{tab:anchors}) \\
\bottomrule
\end{tabular}
\end{table}

\begin{algorithm}[h]
\caption{Batch form. Same relation, same first-occurrence witness, applied to a
finished trajectory instead of to arrivals.}
\label{alg:batch}
\begin{algorithmic}[1]
\Require chunks $x_1{\dots}x_N$ in trajectory order; threshold $t$; step $\varepsilon$
\Ensure retained set $R$; map $\pi$; measured bound $\hat\delta$
\State $\delta \gets \sqrt{2-2t}$;\quad $R \gets \emptyset$;\quad $\hat\delta \gets 0$
\State $e_i \gets \textsc{Encode}(x_i)$ for all $i$ \Comment{one pass, batched}
\For{$i = 1$ \textbf{to} $N$}
  \State $\mathcal{C} \gets \{\, y \in R : \textsc{BandHash}(Q_\varepsilon(e_y)) \cap \textsc{BandHash}(Q_\varepsilon(e_i)) \neq \emptyset \,\}$
  \State $y^{*} \gets \min\{\, y \in \mathcal{C} : \|e_i - e_y\|_2 \le \delta \,\}$
  \If{$y^{*}$ exists} $\pi(x_i) \gets y^{*}$;\ \ $\hat\delta \gets \max(\hat\delta, \|e_i - e_{y^{*}}\|_2)$
  \Else\ \ $R \gets R \cup \{x_i\}$;\ \ $\pi(x_i) \gets x_i$
  \EndIf
\EndFor
\end{algorithmic}
\end{algorithm}

Both algorithms choose the lowest-indexed eligible representative and process
records in trajectory order. They therefore agree for fixed embeddings and
candidate settings when no record is evicted. Eviction is handled separately
because it can remove a representative and invalidate cached context.

\paragraph{Reproducibility and transfer.} The deletion set is byte-identical
across \detNBatches{} batch sizes and across two machines sharing neither
PyTorch version, CUDA version nor GPU architecture (digest \xhwDigestNineFive{}
at $\cos \ge 0.95$ on both \xhwHwA{} and \xhwHwB{}), where SemDeDup agrees with
itself on only \sdStableNineFive{} of its drop set across \sdSeeds{} seeds. On
LHTB the rule removes \lhtbNearNineFiveObs{} of observation tokens against
\nearGreedyNineFive{} here (Figure~\ref{fig:ecr-lhtb}), showing that near-duplicate removal also applies to a second scaffold.

%% file: sections/03-redundancy.tex
\section{Redundancy in agent trajectories}
\label{sec:redundancy}

We measure exact and near-duplicate observations in the SWE-smith
\texttt{tool} split: \corpusRecords{} records, \corpusTraj{} trajectories, and
\corpusTokens{} tokens. All matches are within a trajectory. We tokenize tool
observations with \texttt{cl100k\_base}.

\subsection{L0 granularity: byte-exact duplication}
\label{app:layers}

\begin{table}[t]
\centering
\caption{Byte-exact duplicate tokens by granularity, within-trajectory. The
line-level figure, not the record- or chunk-level one, is the floor any
near-duplicate contribution must be measured above.}
\label{tab:granularity}
\small
\begin{tabular}{llr}
\toprule
granularity & unit & removable \\
\midrule
record     & one whole tool observation        & \exRecord \\
paragraph  & blank-line separated block        & \exPara \\
chunk-512  & fixed 512-token window            & \exChunk \\
line       & a single line                     & \textbf{\exLine} \\
\bottomrule
\end{tabular}
\end{table}

Removal varies by an order of magnitude across granularities
(Table~\ref{tab:granularity}). Paragraph and whole-record hashing give nearly
the same rate because blank lines rarely divide these observations into useful
units. Line hashing removes \exLine{}, compared with \exChunk{} for fixed
chunks. When an agent rereads an edited file, a few changed lines can alter a
whole-record or chunk hash while leaving most line hashes unchanged. We
therefore use line hashing for L0.

\subsection{Near-duplicate headroom}

We embed 512-token chunks with \texttt{bge-base-en-v1.5} and apply the greedy
rule in Section~\ref{sec:rule}. On \subTraj{} trajectories with \subObs{}
records, \subChunks{} chunks, and \subTokens{} tokens, the near-duplicate layer
removes \nearGreedyNineFive{} at $\cos\ge0.95$. Since L0 alone
removes \exLine{} on the full corpus, we evaluate both layers and their
composition on a shared subset in Section~\ref{sec:eval}.

\paragraph{Conversational memory.} On LoCoMo, a conversational memory benchmark, the
same procedure finds \locomoPairs{} near-duplicate pairs out of \locomoTotal{}
($\approx 10^{-5}$). Redundancy at this level is a property of \emph{agent
trajectories}, not of memory benchmarks in general.

\paragraph{Accounting.} Removal is measured as
$\mathrm{tok}(\text{original}) - \mathrm{tok}(\text{residual})$ over whole texts,
never as a sum of per-piece token counts: the tokenizer merges across newlines,
so summing the parts overstates the whole by roughly 5\%.

\subsection{Composition of the two layers and the effect of their order}

\begin{figure}[t]
\centering
\includegraphics[width=\linewidth]{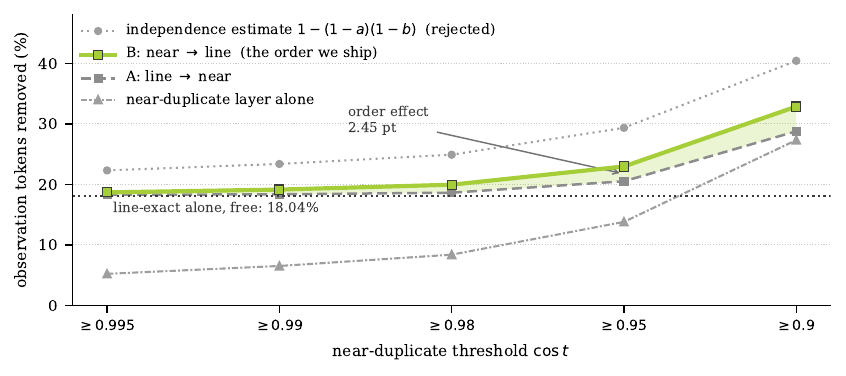}
\caption{Layer composition over \subTraj{} trajectories, weighted by tokens.
Running L0 before L1 reduces joint removal by \ordCostNineFive{}
percentage points at $\cos\ge0.95$. Assuming independence overstates removal
by \indepOverstate{} points. Net joint removal is \jointNet{}, or
\jointGainNet{} points above L0 alone.}

\label{fig:joint}
\end{figure}

The two layers can remove overlapping content, so their combined removal must
be measured. At $\cos \ge 0.95$, their composition removes
\jointGainNet{} percentage points more than L0 alone after accounting for
stubs (Figure~\ref{fig:joint}). Individually, L1 removes
\nearGreedyNineFive{} and L0 removes \lineAlone{}.
Table~\ref{tab:ledger} reports the additional processing cost.

\textbf{Layer order.} L0 followed by L1 removes
\ordANineFive{} at $\cos \ge 0.95$; the reverse order removes \ordBNineFive{}.
The difference is \ordCostNineFive{} percentage points and increases to
\ordCostNinety{} at $\cos \ge 0.90$. Removing repeated lines first fragments
the chunks L1 matches on. LAM therefore runs L1 before L0. The independence estimate $1-(1-a)(1-b)$
overstates measured joint removal by \indepOverstate{} percentage points.

\subsection{Encoder sensitivity}
\label{sec:encsens}

The main evaluation uses one embedding model. To test sensitivity, we repeat
the near-duplicate experiment with four encoders on \encSensChunks{} chunks
from \encSensTraj{} trajectories. This ablation applies L0 first
in all arms to isolate incremental near-duplicate removal. The encoders differ
in training family (bge, e5, gte) and size: MiniLM-L6 has one-third the width
and half the depth of the others. The substitution bound holds in each space;
removal and retained evidence depend on the encoder.

At $t = 0.95$, removal ranges from \encSensMini{} to \encSensEfive{}, a spread
of \encSensSpread{}, compared with \encSensBge{} for bge. Merge thresholds
are not directly comparable across embedding spaces. At this threshold, e5 and
gte recover \encSensRecallFixed{} of bge's deletion set and remove additional
chunks. Matching bge's removal rate instead requires thresholds
$t^\star = \encSensTstarEfive{}$ for e5 and gte and \encSensTstarMini{} for
MiniLM. The resulting Jaccard similarities are \encSensJacGte{},
\encSensJacEfive{}, and \encSensJacMini{}, with deletion-set recall
\encSensRecallGte{}, \encSensRecallEfive{}, and \encSensRecallMini{},
respectively. The overlap indicates similar redundant content across encoders,
but each encoder requires its own threshold calibration.

%% file: sections/05-theory.tex
\section{A query-independent substitution bound}
\label{sec:theory}

\subsection{Ranking margins}

Top-$k$ invariance requires the gap between ranks $k$ and $k+1$
to exceed the relevant score perturbation bound. On \margQ{} LoCoMo queries
against \margRec{} records, the median margin at $k=3$ is \margMedThree{},
well below the substitution bound at useful merge thresholds
(Appendix~\ref{app:margin}). We therefore bound the score deficit of a
substitution without assuming a ranking margin. Retrieval quality is evaluated
separately in Appendix~\ref{sec:realized}.

\subsection{The bound}

Definition~\ref{def:sub} and Theorem~\ref{thm:main} are stated in
Section~\ref{sec:method}; here we prove them and develop what else they imply.

\begin{proof}[Proof of Theorem~\ref{thm:main}]
The left inequality is immediate from $R \subseteq S$. For the right, let $x^*$
attain the maximum over $S$. Since $\pi(x^*) \in R$,
\[
\max_{z\in R} q\cdot z \;\ge\; q\cdot \pi(x^*)
\;=\; q\cdot x^* - q\cdot(x^* - \pi(x^*))
\;\ge\; q\cdot x^* - \|x^* - \pi(x^*)\|_2
\;\ge\; \max_{x\in S} q\cdot x - \delta,
\]
using Cauchy--Schwarz and $\|q\|_2 = 1$. For tightness, take
$S = \{a,b\}$ with $\|a-b\|_2 = \delta$, $R = \{b\}$, $\pi(a) = b$, and
$q = (a-b)/\delta$, a unit vector for which the deficit is exactly $\delta$;
hence no smaller constant is valid. Tightness of $0$ follows from any $\pi$ that
retains the maximiser.
\end{proof}

Theorem~\ref{thm:main} requires no assumption about score distributions or
ranking margins.

\begin{corollary}[top-$k$ $\delta$-cover]
For every unit $q$ and every $k$, each member of the true top-$k$ of $S$ is
either retained in $R$ or has its representative in $R$ within $\ell_2$ distance
$\delta$; its score is reproduced to within $\delta$, uniformly over $q$.
\end{corollary}

\begin{proposition}[rank-wise deficit]
Let $s_j$ and $r_j$ be the $j$-th largest scores in $S$ and in $R$. If $\pi$ is
injective on the true top-$k$, then $0 \le s_j - r_j \le \delta$ for all
$j \le k$.
\end{proposition}
\begin{proof}
$\pi$ maps the top-$j$ set to $j$ distinct elements of $R$, each scoring at
least $s_j - \delta$, so $r_j \ge s_j - \delta$; and $r_j \le s_j$ since
$R \subseteq S$.
\end{proof}

If $\pi$ is not injective on the top-$k$, two collapsed records $x \ne y$ with
$\pi(x) = \pi(y)$ satisfy $\|x-y\|_2 \le 2\delta$: they are near-duplicates of
each other, which is the deletion the method is designed to perform.

\paragraph{The bound read as a rate.} Let $X \sim P$ be a random chunk supported
on $S$. A $\delta$-substitution $(R,\pi)$ is a \emph{deterministic test channel}
$X \mapsto \pi(X)$ whose distortion $\|X - \pi(X)\|_2 \le \delta$ holds almost
surely, hence also in expectation, under every $P$ at once.

\begin{corollary}[distribution-free rate--distortion bound]
\label{cor:rd}
Let $(R,\pi)$ be a $\delta$-substitution of $S$, and let $R_P(D)$ be the
rate--distortion function of a source $X \sim P$ supported on $S$ under
$d(x,\hat x) = \|x - \hat x\|_2$. Then for \emph{every} such $P$,
\begin{equation}
\label{eq:rd}
R_P(\delta) \;\le\; I\!\left(X; \pi(X)\right) \;=\; H\!\left(\pi(X)\right)
  \;\le\; \log_2 |R| \quad \text{bits},
\end{equation}
and the same chain holds under the almost-sure distortion constraint.
\end{corollary}

Because $\pi$ is deterministic, $H(\pi(X)\mid X)=0$, and its output entropy
is at most $\log_2|R|$ (Appendix~\ref{app:proofs}). The pointwise substitution
bound implies \eqref{eq:rd} for every source on $S$. An expected-distortion
bound alone does not imply a guarantee for each query on a finite record set.
The same $\delta$ here bounds both embedding distortion and retrieval-score
deficit. Prior work applies Shannon rate--distortion formulations to prompt
compression \citep{nagle2024limits} and agent memory \citep{zou2026decision}.
We also use metric entropy \citep{kolmogorov1959epsilon,posner1971epsilon} to
bound the size of a retained set.

\begin{corollary}[measured bracket on the rate and on the $\delta$-entropy]
\label{cor:bracket}
Let $H_\delta(S) = \log_2 N(S,\delta)$ be the Kolmogorov $\delta$-entropy of $S$
(\S\ref{sec:entropy}) and let $R_P^{\mathrm{as}}(\delta)$ be the rate at
distortion $\delta$ required \emph{almost surely}. Two runs of greedy
first-occurrence, at radius $\delta$ and at radius $2\delta$, bracket both:
\begin{equation}
\label{eq:bracket}
\log_2 \left|R_{\mathrm{greedy}}(2\delta)\right| \;\le\;
\sup_P R_P^{\mathrm{as}}(\delta) \;\le\; H_\delta(S) \;\le\;
\log_2 \left|R_{\mathrm{greedy}}(\delta)\right|.
\end{equation}
\end{corollary}

The two greedy passes bound both the worst-case rate and metric entropy; they
do not establish equality between them. At $\delta = 0.3162$, the measured
interval is $[\rdLoNF{},\rdHiNF{}]$ bits, compared with \rdLossless{} for the
lossless index. The achieved entropy $H(\pi(X))$ is \rdHNF{}, below
$\log_2|R|$ because cluster masses are unequal. At $\delta = 0.1414$, the
interval has width \rdWidthNN{} bits. The interval widens with $\delta$ as the
second pass merges more records.

\paragraph{From bits to tokens.}
At the default threshold, the index bound decreases from \rdLossless{} to
\rdHiNF{} bits, a reduction of \rdBitSaveNF{} bits (\rdBitSavePctNF{}), while
the same run removes \rdTokSaveNF{} of tokens. These quantities differ because
index rate is logarithmic in retained-set size, whereas context cost scales
with token count. The corresponding token objective is

\begin{equation}
\label{eq:tokdist}
T(\delta) \;\triangleq\; \min_{(R,\pi)\ \delta\text{-sub.\ of } S}
  \ \sum_{z \in R} w(z), \qquad w(z) = \text{tokens of } z,
\end{equation}
a minimum-weight dominating-set problem on the $\delta$-graph of $S$, which
is NP-hard. Coverings of equal size can have different token weights.
Theorem~\ref{thm:ceiling} therefore bounds chunk counts in
Table~\ref{tab:ceiling}; its token columns report measured removal only.

\subsection{A metric-entropy ceiling on removal}
\label{sec:entropy}

Theorem~\ref{thm:main} bounds substitution error. Covering and packing numbers
bound how many records a substitution can remove.

\paragraph{Smallest admissible memory.} Let $N(S,\delta)$ be
the \emph{internal covering number}, the least $|C|$ over $C \subseteq S$ such
that every $x \in S$ has some $c \in C$ with $\|x-c\|_2 \le \delta$, and
$M(S,\delta)$ the \emph{packing number}, the greatest $|P|$ over $P \subseteq S$
with $\|p-p'\|_2 > \delta$ for distinct $p,p'$. Unwinding
Definition~\ref{def:sub}, $(R,\pi)$ is a $\delta$-substitution of $S$
for a suitable map $\pi$ exactly when $R$ is an internal $\delta$-covering.
Thus the maximum admissible removal depends on the geometry of $S$:
\begin{equation}
\label{eq:ceiling}
\rho^{*}(\delta) \;\;\triangleq\;\; \max_{(R,\pi)\ \delta\text{-sub.\ of } S}
  1 - \frac{|R|}{|S|} \;\;=\;\; 1 - \frac{N(S,\delta)}{|S|}.
\end{equation}
No compactor allowing error $\delta$ under Theorem~\ref{thm:main} can remove
more than $\rho^{*}(\delta)$, however it is implemented.

\begin{proposition}[greedy is simultaneously a cover and a packing]
\label{prop:pack}
Let $R$ be produced by greedy first-occurrence at radius $\delta$
(\S\ref{sec:rule}): chunks arrive in order, and a chunk is appended to $R$ iff
it is at distance $>\delta$ from every element already in $R$. Then
$N(S,\delta) \le |R| \le M(S,\delta) \le N(S,\delta/2)$.
\end{proposition}
\begin{proof}
\emph{Cover.} A chunk is dropped only when some element of $R$ at that moment is
within $\delta$ of it, and elements are never removed; so $R$ is an internal
$\delta$-covering and $|R| \ge N(S,\delta)$. \emph{Packing.} A chunk enters $R$
only if it is at distance $>\delta$ from every earlier element, so all pairs in
$R$ are separated by more than $\delta$ and $|R| \le M(S,\delta)$.
\emph{Right-hand side.} Let $C$ be a minimal internal $\delta/2$-covering and
send each $p$ in a $\delta$-packing to a $c \in C$ within $\delta/2$. Two
distinct points sharing a $c$ would be at distance $\le \delta$, contradicting
separation, so the map is injective and $M(S,\delta) \le |C| = N(S,\delta/2)$.
\end{proof}

Both properties hold online without knowing $S$ in advance. We verify the
packing property directly: at $t=0.99$ the
largest cosine between any two surviving chunks across the corpus is
\ceilMaxCos{}, and at every other threshold it likewise falls strictly below $t$.

\begin{theorem}[removal ceiling]
\label{thm:ceiling}
Let $P$ be any $2\delta$-packing of $S$. Then
$\rho^{*}(\delta) \le 1 - |P|/|S|$. In particular, running greedy
first-occurrence at radius $2\delta$ yields such a $P$, so
\begin{equation}
\rho^{*}(\delta) \;\le\; \rho_{\mathrm{greedy}}(2\delta),
\end{equation}
a ceiling that is \emph{computable in one extra pass}. Under the cosine
parameterization $\delta = \sqrt{2-2t}$, radius $2\delta$ is the threshold
$t' = 4t-3$.
\end{theorem}
\begin{proof}
Let $C$ be a minimal internal $\delta$-covering, $|C| = N(S,\delta)$, and map
each $p \in P$ to some $c \in C$ with $\|p-c\|_2 \le \delta$. If distinct
$p,p'$ shared a $c$ then $\|p-p'\|_2 \le 2\delta$, contradicting the
$2\delta$-packing condition; the map is injective, so
$|P| \le N(S,\delta)$ and \eqref{eq:ceiling} gives the claim. Greedy at radius
$2\delta$ returns a $2\delta$-packing by Proposition~\ref{prop:pack}. For the
parameterization, $\sqrt{2-2t'} = 2\sqrt{2-2t}$ gives $t' = 4t-3$.
\end{proof}

Bounding $M$ by volume is too weak: a $\delta$-packing of the unit sphere in
$\mathbb{R}^{D}$ has $\log_2 M \le D\log_2(1+2/\delta)$, which at
$D = \embedDim$ and $\delta = 0.316$ permits $\approx 2^{2200}$ survivors. The
binding constraint is the intrinsic geometry of what agents observe, and
Theorem~\ref{thm:ceiling} makes measuring it a one-pass computation.

\begin{table}[t]
\centering
\caption{The removal ceiling, measured. For each operating point $t$ we run greedy
first-occurrence a second time at $t' = 4t-3$ (radius $2\delta$); by
Theorem~\ref{thm:ceiling} its removal upper-bounds what \emph{any}
$\delta$-bounded compactor could achieve here. Cardinality over \ceilChunks{}
chunks is the bounded quantity. Token columns report observed removal in
the two passes and are not upper bounds on token removal.}
\label{tab:ceiling}
\small
\begin{tabular}{llccccc}
\toprule
& & \multicolumn{3}{c}{chunks removed (cardinality)} &
    \multicolumn{2}{c}{tokens removed} \\
\cmidrule(lr){3-5}\cmidrule(lr){6-7}
$t$ & $\delta$ & attained & ceiling ($t'$) & \% of ceiling & at $t$ & at $t'$ \\
\midrule
$0.99$ & $0.141$ & \ceilAttainedNN & \ceilCeilingNN & \ceilFracNN
       & \ceilTokAttNN & \ceilTokCeilNN \\
$0.95$ & $0.316$ & \ceilAttainedNF & \ceilCeilingNF & \textbf{\ceilFracNF}
       & \ceilTokAttNF & \ceilTokCeilNF \\
$0.90$ & $0.447$ & \ceilAttainedN  & \ceilCeilingN  & \ceilFracN
       & \ceilTokAttN  & \ceilTokCeilN  \\
\bottomrule
\end{tabular}
\end{table}

\paragraph{Attained removal.}
At the default threshold, the near-duplicate layer removes \ceilAttainedNF{}
of chunks against a ceiling of \ceilCeilingNF{}, attaining
\ceilFracNF{} of that upper bound. This is a lower bound on the fraction of
optimal removal achieved, since the ceiling may be loose. The corresponding
fractions are \ceilFracNN{} at $t=0.99$ and \ceilFracN{} at $t=0.90$.
They should be compared between methods at a fixed threshold: increasing
$\delta$ also loosens the second-pass bound. The auxiliary threshold
$t'=4t-3$ is used only to compute the ceiling and is not a deployment setting.

\subsection{Representative selection}
\label{sec:rule}

The merge relation is not transitive: $a$ can be close to $b$ and $b$ to $c$
without $a$ being close to $c$. Keeping one record per connected component of
the threshold graph therefore does not ensure a $\delta$-substitution.
Greedy first-occurrence compares each arrival directly with retained records
and accepts only substitutions within $\delta$.

\begin{figure}[t]
\centering
\includegraphics[width=0.82\linewidth]{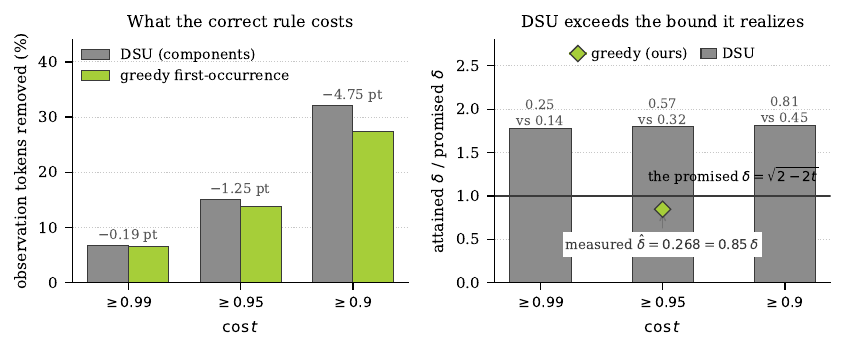}
\caption{Representative selection and the substitution bound.
\textbf{Left:} observation-token removal for connected components (DSU) and
greedy first-occurrence. \textbf{Right:} maximum substitution distance relative
to $\delta = \sqrt{2-2t}$. Greedy remains within the bound. The diamond marks
$\hat\delta = \deltaAttained$ at $\cos \ge 0.95$. Results use \subTraj{}
trajectories and \subTokens{} observation tokens.}

\label{fig:rule}
\end{figure}

Greedy first-occurrence drops $x$ only when an already retained $y$ lies
within $\delta$. This directly satisfies the premise of
Theorem~\ref{thm:main}. Figure~\ref{fig:rule} compares the removal and maximum
substitution distance of this rule with connected-component grouping.

\subsection{The measured bound}

The pipeline reports
$\hat\delta = \max_{x:\pi(x)\ne x}\|x-\pi(x)\|_2 \le \delta$, with value zero
when no substitutions occur. It can be computed in $O(|S|D)$ or maintained
incrementally as records arrive. Theorem~\ref{thm:main} then holds with
$\hat\delta$ in place of $\delta$. For a particular query, the deficit is also
bounded by $\max_i|q\cdot(x_i-\pi(x_i))|$.
On the \subTraj{} trajectories in Section~\ref{sec:eval}, the realized maximum
is \deltaAttained{} at a configured \deltaConfigured{}. This bounds
embedding substitution and score loss. Evidence retention is evaluated
separately in Section~\ref{sec:safety}.

%% file: sections/06-determinism.tex
\section{Determinism}
\label{sec:determinism}

We distinguish identical output bytes from identical deletion sets and test
both across repeated runs, encoder batch sizes, and hardware configurations.

\paragraph{The pipeline by construction.}
For fixed embeddings and arithmetic, quantization, band hashing, and distance
verification are repeatable, and the lowest-index rule removes dependence on
traversal order. Repeated cuSZp runs at relative errors $10^{-3}$, $10^{-2}$,
and $3\times10^{-2}$ give one SHA-256 digest and one output size per setting,
including under concurrent GPU load.

\paragraph{The measured bound.}
Measured cosines vary by \cosRepro{}, which $\delta = \sqrt{2-2\cos}$ amplifies
near one into a \deltaSpread{} spread in $\hat\delta$. Online and batch
execution nevertheless produce byte-identical deletion sets on
\onlineEqBatch{} trajectories.

\paragraph{The encoder.}
We chunk \detRecords{} records from \subTraj{} trajectories into \detChunks{}
pieces, encode them at \detNBatches{} batch sizes (\detBatches{}), run the
greedy rule on each embedding matrix, and compare SHA-256 digests. Relative to
batch 64, coordinates vary by up to \detJitterEmb{} and pairwise cosines by
\detJitterCos{}, yet every threshold yields one digest across all batch sizes:
\detDropNineNine{} chunks deleted at $\cos \ge 0.99$, \detDropNineFive{} at
$\cos \ge 0.95$, Jaccard \detJaccard{}, symmetric difference \detSymDiff{} on
\detHw{}.

\paragraph{Decision margins.}
\label{sec:eval-determinism}
Over \hrDecisions{} decisions the margin $|\max_j \cos(x_i,x_j)-t|$ is at
minimum \hrMinNineNine{} at $t = 0.99$ and \hrMinNineFive{} at $t = 0.95$,
\hrRatioNineNine{} and \hrRatioNineFive{} times the measured jitter; the first
percentile is \hrPOneNineNine{}, the median \hrMedNineNine{}, and
\hrWithinTenX{} of decisions lie within ten times the jitter. Margins this wide
explain the unchanged deletion sets.

\paragraph{SemDeDup seed sensitivity.}
Holding the \detChunks{} chunks and their order fixed and varying the $k$-means
seed over \sdSeeds{} runs with \sdClusters{} clusters, removal at $\cos \ge
0.99$ is \sdRateNineNine{} with spread \sdRateSpread{}, but pairwise
deletion-set Jaccard is only \sdJaccardNineNine{} at $0.99$ and
\sdJaccardNineFive{} at $0.95$: all seeds share \sdStableNineNine{} and
\sdStableNineFive{} of the deletion union, and \sdFlipNineNine{} chunks change
status. Similar removal rates therefore need not imply identical memories.

\paragraph{Cross-hardware.}
Both machines use the same checksum-verified parquet shard, tokenizer, fp32 CLS
pooling, and greedy rule, differing in PyTorch version, CUDA version, and GPU
architecture. \xhwVerdict{}

\paragraph{LLM memory managers.}
Submitting a Mem0-style prompt that lists records and requests deletion indices
to \lldModel{} under vLLM with prefix caching disabled: \llmdetStatus{} Of
\lldTraj{} trajectories \lldEmpty{} yield empty deletion sets, and the
\lldNonEmpty{} nonempty ones have sizes \lldSizes{}. Reproducibility here rests
on greedy decoding, a fixed seed, and a pinned model version rather than on
embeddings and a merge threshold.

%% file: sections/07-cost.tex
\section{Cost model}
\label{sec:cost}

Appendix~\ref{app:costclasses} charges each compaction method for inference,
compression, and token removal. Here we model end-to-end latency and evaluate
scheduling and CPU placement.

\subsection{The $\varphi$ decomposition (P1)}
\label{sec:phi}

Split one agent turn into the part scaling with context length $C$ --- KV reads
during decode, prefill of arriving tokens against the existing context --- and
the part that does not, dominated by streaming the weights once per decode step.
With $\varphi$ the first part's share of turn latency, context reduction cannot
beat the Amdahl ceiling $1/(1-\varphi)$.

\begin{table}[t]
\centering
\caption{$\varphi$, the context-scaling share of one turn, and the ceiling it
puts on \emph{any} compactor. Llama-3.1-8B on GB10, $\Delta = \turnDelta{}$
arriving tokens and $G_a = \turnGa{}$ decoded tokens per turn (measured).
Weight reads amortize across a decode batch; KV reads do not.}
\label{tab:phi}
\small
\begin{tabular}{lrrrr}
\toprule
 & $C = 3$K & $C = 8$K & $C = 32$K & $C = 128$K \\
\midrule
batch 1  & \phiSmall\ (\ceilSmall) & 6.4\%\ (1.07$\times$) & \phiOne\ (\ceilOne) & 52.4\%\ (2.10$\times$) \\
batch 32 & 23.5\%\ (1.31$\times$)  & 45.0\%\ (1.82$\times$) & \phiThirtyTwo\ (\ceilThirtyTwo) & \phiBig\ (\ceilBig) \\
\bottomrule
\end{tabular}
\end{table}

\begin{table}[t]
\centering
\caption{Modeled latency with the same compressor, corpus, and deletion set.
Only the schedule differs. Delete-in-place prefills the arriving tokens and then
removes duplicates; the $\varepsilon$ error-bounded rule removes before prefill.
Llama-3.1-8B, growth regime, \subTraj{} trajectories.}
\label{tab:policy}
\footnotesize
\begin{tabular}{lrrr}
\toprule
policy & batch 1 & batch 32 & \% of total \\
\midrule
delete in place (prefill, then remove)      & $\polDelete$\,s & $\polDeleteBs$\,s & \polShareDelete \\
\shiprow $\varepsilon$ error-bounded (remove, then prefill)   & $\polAdmit$\,s  & $\polAdmitBs$\,s  & \polShareAdmit \\
\midrule
policy gain & \best{\polGainOne} & \best{\polGainBs} & \\
\bottomrule
\end{tabular}
\end{table}

At batch 1 and a 32K context $\varphi = \phiOne{}$, capping context-reduction
speedup at \ceilOne{}; at batch 32 the same context gives
$\varphi = \phiThirtyTwo{}$ and \ceilThirtyTwo{} (Table~\ref{tab:phi}). Batch
size, context length, and model architecture therefore move the speedup on
their own, before any compactor acts.

\paragraph{Accounting behind Table~\ref{tab:e2emethods}.} The end-to-end table
in \S\ref{sec:main-results} runs the loop above for \eeTurns{} turns in the
growth regime at $B_{ctx} = 32$K, charging each method its own measured removal
and the least expensive schedule its rule permits. Specifically, a method that
is not prefix-stable can still defer compaction to the budget wall instead of
deleting eagerly, and that cheaper option is what it is charged; only LAM and L0
admit. Furthermore the ordering there does not follow removal, because what each
method pays is set by how often its rule forces a re-prefill rather than by how
much it removes. Compaction latency itself is a separate and much smaller axis
(Appendix~\ref{app:costclasses}). Without compaction the run exhausts the KV
cache at turn \eeNoneTurns{} of \eeTurns{}.

\paragraph{End-to-end accounting.} Without overlap, total latency includes
compaction and inference on the retained context:
\[
t_{e2e} = \sum_{\text{turns}} \left[\text{prefill} + \text{decode}\right]
        + \sum_{\text{events}} \left[\text{compact} + \text{re-prefill}\right].
\]
Over \subTraj{} measured trajectories a turn adds $\Delta = \turnDelta{}$ tokens,
\turnFobs{} of it tool observation, while the agent \emph{decodes} only
$G_a = \turnGa{}$: most arriving tokens are eligible for compaction, and what
their removal is worth depends on $\varphi$.

\subsection{Where the latency is: prefix-stability (S1)}
\label{sec:policy}

We hold the compressor, corpus, and deletion set fixed and vary only when
records are tested (Table~\ref{tab:policy}). Deleting a record from an existing
context invalidates the KV entries after it, and rebuilding a 360K-token cache
costs ${\approx}\reprefillCost$\,s in the model against
${\approx}\turnServeCost$\,s for the turn's inference. Deletion in place
therefore spends \polShareDelete{} of modeled runtime on compaction and
re-prefill; testing before prefill preserves the cache and cuts that to
\polShareAdmit{}, for \polGainOne{} end to end at batch 1 and \polGainBs{} at
batch 32.

\paragraph{The re-prefill term, measured.}
Since \reprefillCost{} comes from the cost model, we also measure prefill
directly for \ovModel{} under vLLM on \ovGpus{} with prefix caching enabled,
issuing a salted prompt and then repeating it unchanged at each context length
up to \ppMaxLen{} tokens. The server reports \ppCachedFrac{} of the repeat as
cached, and at the longest context cold and cached prefill take \ppColdS{} and
\ppWarmS{}, a \ppWarmSpeedup{} difference --- independent confirmation that
preserving the cached prefix cuts prefill cost. The cost is also nonlinear in
context length: the effective rate peaks at \ppRateHi{} near 4K tokens and falls
to \ppRateLo{} at \ppMaxLen{}, and $t=c+C/r+kC^2$ fits better than a line,
predicting \ppReprefill{} at 360K tokens against \ppReprefillLin{} for the
linear fit (\ppLinUnderstate{} apart). Both extrapolations lie outside the
measured range. They also sit \reprefillVsMeasured{} below the
\reprefillCost{}\,s the cost model charges, because the two are different
stacks: the model is configured for Llama-3.1-8B on one \detHw{} at
\hwGbBw{} of achieved bandwidth, while this measurement runs \ovModel{} on
\ovGpus{} at \hwAmpereBw{} each. Aggregate bandwidth alone separates them by
\bwRatioAgg{}, so the model's figure is below what hardware scaling would
predict rather than above it. Neither stack is measured at 360K tokens, and
the headline rests on the model's.

\paragraph{How much a deletion invalidates.}
Deleting in place is not a competing system; it is the same operator on a
different schedule, and the cost it pays turns on how much of the KV cache a
deletion voids. That is measurable rather than assumed. A deletion invalidates
every entry at or after the earliest deleted position, and a duplicate by
definition repeats something earlier, so the first one arrives early: across
\invalNTraj{} trajectories the median first duplicate sits \invalFirstDup{} of
the way in, and the invalidated share of context is \invalMean{} on average and
\invalMedian{} at the median, with \invalGeNF{} of trajectories above $0.95$.
Measured on exact hashing, the least aggressive layer, so these are lower
bounds. Charging the whole context is therefore close to what the schedule
actually costs; recomputing at the measured mean instead gives
\polGainOneMeas{}--\polGainBsMeas{} rather than \polGainOne{}--\polGainBs{}.

\paragraph{The three factors together.}
In the evaluated regime the removal L1 adds over L0 is worth \growthGain{},
kernel optimization almost nothing because encoding dominates compaction time,
and avoiding re-prefill \polGainOne{}--\polGainBs{}. That last gain goes to any
prefix-preserving rule, L0 included; summarization rewrites existing context and
must rebuild the affected cache entries instead.

\paragraph{Index footprint.}
The index holds one fp32 embedding per retained chunk and one $64$-bit hash per
distinct line. Across \dhTraj{} trajectories at $\cos \ge 0.95$ it is
\idxMeanMB{} on average, \idxPNineNineMB{} at the $99$th percentile and
\idxMaxMB{} at maximum, or \idxPerKtok{} per $1$K observation tokens, of which
embeddings are \idxShare{}. A log--log fit of bytes against tokens gives an
exponent of \idxGrowth{}, sublinear over the measured range. The index is
\idxOverKV{} of the corresponding \idxKVPerTraj{} KV cache for \ovModel{}.

\subsection{Where the compactor runs}
\label{sec:overlap}

The test can run on the CPU while the GPU serves the model, over embeddings that
are already available. We measure whether that overlap costs serving throughput
or latency.

\paragraph{Setup.} Two nodes of ALCF sirius: node~1 serves \ovModel{} under vLLM
across \ovGpus{} at tensor parallelism \ovTP{} and runs the compactor on its
\ovCores{} CPU cores, node~2 runs the load generator. Prefix caching is disabled
so that reuse of the shared prompt prefix cannot hide the interference. The five
conditions are the server alone (A), the compactor alone (B), both overlapped
under saturating load (C) and at a fixed arrival rate of \ovRate{} (D), and (D)
with the compactor confined by \texttt{numactl} to the two of \ovNuma{} NUMA
domains furthest from the GPUs (E). Compactor rates cover the whole benchmark
window, so idle time counts against them.

\begin{table}[t]
\centering
\caption{CPU compaction overlapped with GPU serving on two sirius nodes. At 32
threads the compactor contends with the server's own CPU work and both sides
lose.}
\label{tab:overlap}
\small
\input{tables/overlap}
\end{table}

\paragraph{Throughput and latency under overlap.}
At the fixed arrival rate, output throughput is \ovRateTputEight{} of baseline
with 8 compactor threads and \ovRateTputSixteen{} with 16; median TTFT moves
from \ovRateBaseTtft{} to \ovRateTtftSixteen{} and median inter-token latency by
\ovRateTpotSixteen{}. The compactor processes \ovCompactRateSixteen{} context
tokens per second against \ovSoloSixteen{} alone, \ovVsSoloSixteen{} of its
standalone throughput. Under saturating load, serving throughput is
\ovSatTputSixteen{}, TTFT \ovSatTtftSixteen{}, and the compactor
\ovSatVsSoloSixteen{} of standalone. At 32 threads TTFT rises to
\ovRateTtftThirtytwo{}, inter-token latency to \ovRateTpotThirtytwo{}, and
compactor throughput falls to \ovVsSoloThirtytwo{} --- CPU oversubscription,
with compaction sharing cores with server tokenization and scheduling.

\paragraph{NUMA placement.}
Confining compaction to the two NUMA domains furthest from the GPUs gives server
throughput \ovPinTput{} and compactor throughput \ovPinVsSolo{} of standalone:
unrestricted placement wins, since pinning costs the compactor memory bandwidth.

\paragraph{Consequences for the cost model.}
At the selected setting the CPU processes context at roughly \ovKeepUp{} the
GPU's output-token rate, so this workload overlaps with serving at little
throughput cost, and overlap reduces what compaction contributes to elapsed time
below what the sequential model charges. The result assumes spare CPU capacity
and excludes GPU encoding when embeddings must be generated.

%% file: tables/overlap.tex
\begin{tabular}{lrrrrr}
\toprule
compactor & \multicolumn{3}{c}{GPU server} & \multicolumn{2}{c}{CPU compactor} \\
\cmidrule(lr){2-4}\cmidrule(lr){5-6}
threads & out tok/s & TTFT p50 & TPOT p50 & ctx tok/s & vs.\ solo \\
\midrule
\multicolumn{6}{l}{\emph{fixed arrival rate, 5 req/s --- server unsaturated}} \\
none & 990.8 & 157.0\,ms & 11.81\,ms & --- & --- \\
8    & 990.8 & 164.7\,ms & 11.84\,ms & 2{,}937 & 0.966 \\
16   & 990.7 & 166.5\,ms & 11.84\,ms & 4{,}647 & 0.929 \\
32   & 989.0 & 192.5\,ms & 13.68\,ms & 3{,}809 & 0.620 \\
\midrule
\multicolumn{6}{l}{\emph{saturating load --- server queue-bound}} \\
none & 1267.9 & 100.3\,s & 77.18\,ms & --- & --- \\
8    & 1266.2 & 100.4\,s & 77.26\,ms & 2{,}908 & 0.946 \\
16   & 1266.6 & 100.3\,s & 77.25\,ms & 4{,}536 & 0.907 \\
32   & 1258.5 & 101.4\,s & 77.34\,ms & 3{,}747 & 0.610 \\
16, pinned & 1263.6 & 100.7\,s & 77.32\,ms & 2{,}703 & 0.540 \\
\bottomrule
\end{tabular}

%% file: sections/B-eval-extended.tex
\section{Extended evaluation}
\label{app:eval}

\subsection{The online operator, measured (A1)}

\begin{table}[t]
\centering
\caption{The $\varepsilon$ error-bounded test over \subTraj{} SWE-smith trajectories, \subObs{} observation records, \subTokens{} tokens, $\cos \ge 0.95$.}
\label{tab:online}
\small
\begin{tabular}{llp{6.6cm}}
\toprule
quantity & value & note \\
\midrule
gross removal        & \jointGross & L1 followed by L0 \\
net of stubs         & \jointNet   & stubs cost \stubCost{} of input \\
chunks encoded       & \subChunks  & one encoding per chunk \\
per-turn latency     & \turnLatMean{} mean / \turnLatPnn{} p99 & a turn itself costs \turnCostLo{}--\turnCostHi{} \\
measured bound       & $\hat{\delta} = \deltaAttained \le \deltaConfigured$ & attained, not configured \\
online $\equiv$ batch & byte-identical residuals & verified deletion-set equality, \onlineEqBatch{} pass \\
\bottomrule
\end{tabular}
\end{table}

\begin{table}[t]
\centering
\caption{Substitution distances and retrieval quality. Distance summaries are weighted by deleted tokens; $\hat\delta$ is the maximum within a run. \emph{Keep} is the retained-chunk fraction. R@10 compares with the full index's top-10 after applying the substitution map.}
\label{tab:delta-hat}
\small
\input{tables/delta_hat}
\end{table}

\begin{table}[t]
\centering
\caption{Layer composition and order, $\cos \ge 0.95$, same denominator. The independence estimate is compared with measured joint removal.}
\label{tab:layers}
\small
\begin{tabular}{lrr}
\toprule
configuration & removal & vs.\ default \\
\midrule
L0 alone (exact hashing)            & \lineAlone       & $-\,4.93$\,pt \\
L1 alone (near-duplicate)                  & \nearGreedyNineFive & $-\,9.16$\,pt \\
A: L0 $\to$ L1                             & \ordANineFive    & $-\,\ordCostNineFive$\,pt \\
B: L1 $\to$ L0 (default)                    & \ordBNineFive  & --- \\
\midrule
independence estimate (rejected)           & \ordIndepNineFive & $+\,6.38$\,pt \\
\bottomrule
\end{tabular}
\end{table}

\begin{table}[t]
\centering
\caption{Per-stage cost of one full compaction pass at trajectory scale.}
\label{tab:ablation}
\small
\begin{tabular}{lrr}
\toprule
stage & seconds & \% of total \\
\midrule
bge-base encode        & 13.945508 & 99.989\% \\
$Q_\varepsilon$ quantize & 0.000138 & 0.001\% \\
band digests           & 0.000078 & 0.001\% \\
exact $\ell_2$ verify  & 0.001201 & 0.009\% \\
cuSZp byte layer       & 0.000069 & 0.000\% \\
\midrule
total                  & 13.946994 & \\
\bottomrule
\end{tabular}
\end{table}

Table~\ref{tab:online} reports the greedy rule with reference-stub costs
included, and online and batch residuals are byte-identical on \onlineEqBatch{}
trajectories. Running L0 before L1 costs \ordCostNineFive{} percentage points of
joint removal (Table~\ref{tab:layers}), and encoding accounts for nearly all
compaction time (Table~\ref{tab:ablation}), so reusing existing embeddings is
worth far more than optimizing the remaining stages.

\subsection{Retrieval quality and the measured bound}
\label{sec:realized}

Recording $d_i=\|x_i-x_{r(i)}\|_2$ for every substitution on \dhTraj{}
trajectories (\dhChunks{} chunks, \dhTokens{} observation tokens), all satisfy
the configured bound and the largest approach it at each threshold: at
$\cos \ge 0.99$ the mean distance is \dhRatioNineNine{} of the bound and
\dhHalfNineNine{} of merged tokens fall below half of it.

\paragraph{Retrieval cost.} Indexing each trajectory's observation chunks and
querying with its own action texts (\rkQueries{} queries over \rkChunks{}
chunks), the default $\cos \ge 0.95$ keeps \rkKeepNineFive{} of chunks at
recall@10 \rkTenNineFive{} (recall@1 \rkOneNineFive{}, recall@20
\rkTwentyNineFive{}) and $\cos \ge 0.99$ gives \rkTenNineNine{}. At $\cos \ge
0.90$, keeping \rkKeepNine{}, recall@10 drops to \rkTenNine{} --- the threshold
at which the resolve-rate ceiling also falls (Section~\ref{sec:rdu}).

\subsection{Evidence survival (A4)}
\label{sec:safety}

\paragraph{Ground truth and criterion.} SWE-smith applies a bug patch to a
working repository, so the patch's \texttt{+} side is the defect the agent
explores and the gold fix is that patch reversed. The patch comes from the task
instance, not the misaligned \texttt{patch} column of the trajectory row:
joining on \texttt{instance\_id} matches \safJoined{} of \subTraj{}, and the
\texttt{+} side appears in \safPlusSide{} of instances against \safMinusSide{}
for the \texttt{-} side. Evidence is the added plus context lines of at least 12
characters, so that \texttt{else:} does not count. We score the \safLines{} such
lines appearing verbatim in some observation over \safTraj{} trajectories ---
compaction cannot delete what the agent never saw --- and a line survives if it
is findable anywhere in the residual.

\begin{table}[t]
\centering
\caption{Gold-patch safety, \safTraj{} trajectories, \safLines{} witnessed evidence lines. ``Kept'' pools over lines (micro retention); ``intact'' is the fraction of trajectories losing none (macro). The no-compaction and L0 rows are harness checks and must read $100.000\%$; composing L1 with L0 must score what L1 scores alone.}
\label{tab:safety}
\small
\begin{tabular}{lrrrr}
\toprule
method & removal & evidence kept & trajectories intact & lines lost \\
\midrule
no compaction                   & 0.00\%  & 100.000\% & 100.0\% & 0 \\
L0 alone (exact hashing) & 18.02\% & 100.000\% & 100.0\% & 0 \\
\midrule
ours, $\cos \ge 0.99$           & 18.86\% & 100.000\% & 100.0\% & 0 \\
ours, $\cos \ge 0.96$           & 21.49\% & 100.000\% & 100.0\% & 0 \\
\shiprow ours, $\cos \ge 0.95$  & 22.47\% & 99.984\% & 99.8\% & 1 \\
ours, $\cos \ge 0.90$           & 31.99\% & 98.709\%  & 93.0\%  & 81 \\
\midrule
observation masking, $k=50$     & 4.42\%  & 99.140\%  & 98.3\%  & 54 \\
observation masking, $k=20$     & 34.06\% & 93.021\%  & 81.9\%  & 438 \\
observation masking, $k=10$     & 67.38\% & 80.752\%  & 52.2\%  & 1{,}208 \\
observation masking, $k=5$      & 77.94\% & 73.136\%  & 35.6\%  & 1{,}686 \\
observation masking, $k=3$      & 81.44\% & 71.606\%  & 32.8\%  & 1{,}782 \\
\bottomrule
\end{tabular}
\end{table}

\begin{table}[t]
\centering
\caption{Ratio ledger, \subTraj{} trajectories, \subTokens{} observation tokens, $\cos \ge 0.95$. Context gain is $1/(1-r)$; ``share of turn'' is per-event cost against a turn costing \turnRefCost{}. ``Net'' charges us for reference stubs.}
\label{tab:ledger}
\footnotesize
\begin{tabular}{lrrrr}
\toprule
method & removal & context gain & cost / event & share of turn \\
\midrule
L0 alone (xxHash3)           & \lineAlone          & \ctxGainLine  & \costPerEventHash & $3.1\times10^{-6}$ \\
L1 alone (near-duplicate)          & \nearGreedyNineFive & \ctxGainNear  & \costPerEventOurs & 0.0089 \\
L1 $\to$ L0, gross                 & \jointGross         & \ctxGainGross & \costPerEventOurs & 0.0089 \\
L1 $\to$ L0, net                   & \jointNet           & \ctxGainOurs  & \costPerEventOurs & 0.0089 \\
\midrule
ours $-$ hashing                      & $+\jointGainNet$\,pt & \ctxGainDelta & \costRatioHash & --- \\
\bottomrule
\end{tabular}
\end{table}

At $\cos \ge 0.95$ LAM removes \safOursRemoval{} of observation tokens and
retains \safOursMicro{} of measured evidence, losing \safOursLost{} line in
\safOursLost{} of \safTraj{} trajectories, while keeping the last 50
observations removes \safRemovalRatio{} fewer tokens and loses \safLossRatio{}
more evidence. At matched removal, masking with $k=20$ removes
\safMaskTwentyRemoval{} and retains \safMaskTwentyMicro{} against
\safOursNinetyRemoval{} and \safOursNinetyMicro{} for LAM at $\cos \ge 0.90$.
The \safResolvedTraj{} originally resolved trajectories behave the same
(\safResolvedRemoval{} removal, \safResolvedMicro{} retention).

\paragraph{Stub exposure and specificity.} A line survives only if its
\emph{content} is readable, never because a stub points at it. At $\cos \ge
0.95$ the near-duplicate layer suppressed \safStubTouched{} witnessed lines from
the record they arrived in, of which \safStubRecovered{} remain readable
elsewhere and \safStubLost{} does not. The test is specific: searched in an
unrelated trajectory's memory at a fixed offset, only \safCtrlOther{} of the
\safLines{} are found. This shows the evidence is still in memory, not that the
agent would still produce the patch (Section~\ref{sec:limitations}).

Adding L1 above L0 buys \jointGainNet{} percentage points of net removal at
\costRatioHash{} times L0's per-event cost (Table~\ref{tab:ledger}); what that
is worth end to end is modeled in Appendix~\ref{sec:arrival}.

\subsection{What the evidence axis includes}
\label{sec:relaxed}

Exact substring matching cannot score token-dropping compressors, so we relax
it: a witnessed line survives if its alphanumeric tokens occur in order within a
window at most \relaxDilation{} times their count. Under the same
unrelated-trajectory control, \relaxOwn{} of lines match their own trajectory
and \relaxCtrl{} another, against \safCtrlOther{} for substring matching; that
control rate rises from \relaxCtrlOne{} without dilation to \relaxCtrlEight{} at
eightfold, so we use fourfold.

This covers deletion and token dropping, including LLMLingua-2
\citep{pan2024llmlingua2}, which LightMem uses at intake
\citep{fang2026lightmem}; we evaluate LightMem's configured
\texttt{bert-base-multilingual-cased-meetingbank} checkpoint, not the
XLM-R-large one modeled in Appendix~\ref{app:costclasses}, and not its full
consolidation system. Neither criterion reaches the paraphrases or latent
representations of MemAgent \citep{yu2026memagent}, Gisting, ICAE, or xRAG,
which need a semantic evaluation of their own.

\subsection{The balance: one metric for evidence and tokens (A5)}
\label{sec:balance}

Removal alone does not measure retained utility --- masking with $k=5$ removes
\safMaskFiveRemoval{} of tokens and loses about a quarter of the measured
evidence --- so rather than weight removal against loss we report maximum
removal subject to a retention floor $\tau$:

\begin{equation}
\mathrm{ECR}@\tau \;=\; \max\bigl\{\, \rho(m) \;:\; m \in \mathcal{M},\;
  \mathrm{micro}(m) \ge \tau \,\bigr\},
\end{equation}
over a method's own operating points $\mathcal{M}$ (our $t$, masking's $k$).

\paragraph{Deriving the floor.}
Pooled retention hides trajectory-level loss: SemDeDup at $\cos 0.995$ retains
\sdSafeMicroTightish{} of lines, yet \floorDmgNN{} of trajectories lose at least
one. Across \floorPoints{} settings from \floorFamilies{} families, raising the
floor from $0.990$ to $0.999$ cuts the largest admissible fraction of affected
trajectories from \floorDmgNN{} to \floorDmgNNN{} and the largest RSC decrease
from \floorRscNN{} to \floorRscNNN{} percentage points (Table~\ref{tab:floor});
against an uncompacted resolve rate of \rscBase{} on \rduResolved{} of
\rduTraj{} trajectories (binomial standard error \floorSE{} points), those are
about one standard error and one-tenth of one. We therefore use $\tau=0.999$,
chosen after examining the data. The transition lies between \floorCutFail{},
where the least intact setting (LAM at $t=0.93$) retains all evidence in
\floorFailMacro{} of trajectories, and \floorCutHold{}, the loosest tested floor
meeting the $99\%$ requirement.

\begin{table}[t]
\centering
\caption{Maximum evidence loss allowed by each floor across \floorPoints{} settings from \floorFamilies{} families. ``Tasks losing evidence'' is $1-\mathrm{macro\_all\_intact}$. ``RSC decrease'' is the largest percentage-point decrease from \rscBase{} among the two families with measured RSC.}
\label{tab:floor}
\footnotesize
\begin{tabular}{lrllrl}
\toprule
floor $\tau$ & admissible & evidence loss & binding point & RSC drop & binding point \\
\midrule
$1.000$ & 7  & 0.0\%  & (all tie)              & $-$0.00 pp & (all tie) \\
$0.999$ & 10 & 0.7\%  & SemDeDup $\cos 0.9999$ & $-$0.17 pp & ours $t{=}0.95$ \\
$0.990$ & 16 & 5.5\%  & SemDeDup $\cos 0.995$  & $-$1.83 pp & ours $t{=}0.92$ \\
$0.950$ & 22 & 12.0\% & ours $t{=}0.88$        & $-$8.50 pp & ours $t{=}0.88$ \\
\bottomrule
\end{tabular}
\end{table}

\begin{table}[t]
\centering
\caption{Removal and evidence in one place, same corpus and denominator. $\mathrm{ECR}@\tau$ is the most a family can remove while retaining $\ge \tau$ of the gold-patch evidence lines, maximized over its own operating points; ``best'' drops the constraint. Shading marks the column winner. SemDeDup is swept to $\cos 0.9999$, tighter than our own family.}
\label{tab:ecr}
\footnotesize
\begin{tabular}{lrrrrr}
\toprule
 & best & \multicolumn{4}{c}{$\mathrm{ECR}@\tau$} \\
\cmidrule(lr){2-2}\cmidrule(lr){3-6}
method & (unconstrained) & $\tau = 1.000$ & $0.999$ & $0.990$ & $0.950$ \\
\midrule
L0 alone             & 18.02\% & 18.02\% & 18.02\% & 18.02\% & 18.02\% \\
near-duplicate layer alone & 26.09\% & 11.42\% & 13.08\% & 13.08\% & 26.09\% \\
observation masking        & 81.44\% & ---     & ---     & 4.42\%  & 4.42\%  \\
SemDeDup, 5 seeds          & 47.19\% & ---     & 19.60\% & \best{22.84\%} & \best{33.74\%} \\
\midrule
ours: near $\to$ line      & 31.99\% & \best{21.49\%} & \best{22.47\%} & 22.47\% & 31.99\% \\
\bottomrule
\end{tabular}
\end{table}

At $\tau=0.999$ LAM removes the most of any evaluated setting
(Table~\ref{tab:ecr}): SemDeDup reaches \sdEcrNNN{} and no masking setting
qualifies. At perfect retention LAM removes \ecrOneOurs{}, while SemDeDup's
tightest tested threshold, $\cos 0.9999$, retains \sdSafeMicroTight{} and does
not qualify. At lower floors SemDeDup removes more: \sdEcrNN{} against
\ecrNNOurs{} at $0.99$ and \sdEcrNF{} against \ecrNFOurs{} at $0.95$, although
it clusters across trajectories here (Section~\ref{sec:sdwithin} restricts it to
one). Which family is preferable depends on the required retention; only LAM
carries a substitution bound and admits testing before prefill.

\paragraph{Turning loss into cost.} Charging lost evidence $\kappa$ repair turns
at $\Delta = \costDelta$ measured tokens each (\S\ref{sec:cost}), the \emph{net}
reduction is
\begin{equation}
\label{eq:eta}
\eta(\kappa) \;=\; \rho \;-\; \frac{\text{repairs} \cdot \kappa \cdot \Delta}{T},
\qquad T = \costTokens \text{ observation tokens},
\end{equation}
and we sweep $\kappa$ rather than choose it.

\begin{table}[t]
\centering
\caption{Net reduction $\eta(\kappa)$ from \eqref{eq:eta}, assuming $\kappa$ repair turns per lost evidence line. Break-even values show when repair costs offset additional removal.}
\label{tab:eta}
\small
\begin{tabular}{lrrrrr}
\toprule
& $\rho$ & $\eta(1)$ & $\eta(3)$ & $\eta(8)$ & break-even $\kappa$ vs.\ ours \\
\midrule
ours, $\cos\ge0.95$ & 0.2247 & 0.2246 & 0.2245 & 0.2241 & --- \\
ours, $\cos\ge0.90$ & 0.3199 & 0.3141 & 0.3027 & 0.2741 & 16.9 \\
masking, $k{=}20$   & 0.3406 & 0.3097 & 0.2478 & 0.0932 & \textbf{3.8} \\
masking, $k{=}10$   & 0.6738 & 0.5884 & 0.4178 & $-0.0087$ & \textbf{5.3} \\
masking, $k{=}5$    & 0.7794 & 0.6603 & 0.4222 & $-0.1732$ & \textbf{4.7} \\
\bottomrule
\end{tabular}
\end{table}

\begin{table}[t]
\centering
\caption{Budget-feasibility rate $\mathrm{BFR}@B$: the fraction of \bfrTraj{} trajectories whose observation stream fits under $B$ tokens. Necessary, not sufficient.}
\label{tab:bfr}
\small
\begin{tabular}{lrrrr}
\toprule
& $B{=}8$K & $B{=}16$K & $B{=}32$K & $B{=}64$K \\
\midrule
uncompacted          & \bfrNoneA & \bfrNoneB & \bfrNoneC & \bfrNoneD \\
ours, $\cos\ge0.99$  & \bfrNNA   & \bfrNNB   & \bfrNNC   & \bfrNND   \\
ours, $\cos\ge0.95$  & \bfrNFA   & \bfrNFB   & \bfrNFC   & \bfrNFD   \\
ours, $\cos\ge0.90$  & \bfrNA    & \bfrNB    & \bfrNC    & \bfrND    \\
\bottomrule
\end{tabular}
\end{table}

Masking keeps a net-token advantage while repair is cheap, breaking even against
LAM at around four repair turns per lost line depending on $k$ and turning
negative near $\kappa=\breakevenZero$, or $\kappa=\breakevenTraj$ charging one
repair per affected trajectory (Table~\ref{tab:eta}). LAM at $\cos\ge0.95$ loses
one measured line, so the sweep barely moves its net reduction.

\paragraph{Cost in tokens and dollars.} Removing \safOursRemoval{} of
\costTokens{} observation tokens is \costSavedPerK{}M context tokens per
$1{,}000$ trajectories, or $\costSavedPerK p$ at $p$ per million input tokens.
That under-counts: without prompt caching a context token is re-read on every
later turn, so at \turnsMean{} turns per trajectory the saving is amplified by
$\approx n/2 =$ \costAmp{}. We report the multiplier rather than a dollar
figure, which would need a cache policy and a price list.
\citet{johnson2026production} make the same point from the other side:
aggressive compression \emph{increased} total cost by $1.8\%$ despite large
input reduction, because output tokens expanded and are priced higher. We do not
measure output length.

\paragraph{Budget feasibility.}
$\mathrm{BFR}@B$ is the fraction of recorded observation streams fitting within
$B$ tokens after compaction --- storage feasibility, not task success. The
default setting raises it from \bfrNoneB{} to \bfrNFB{} at $B=16$K and from
\bfrNoneC{} to \bfrNFC{} at $32$K; at $64$K nearly all streams already fit.

\begin{figure}[!htbp]
\centering
\includegraphics[width=0.95\linewidth]{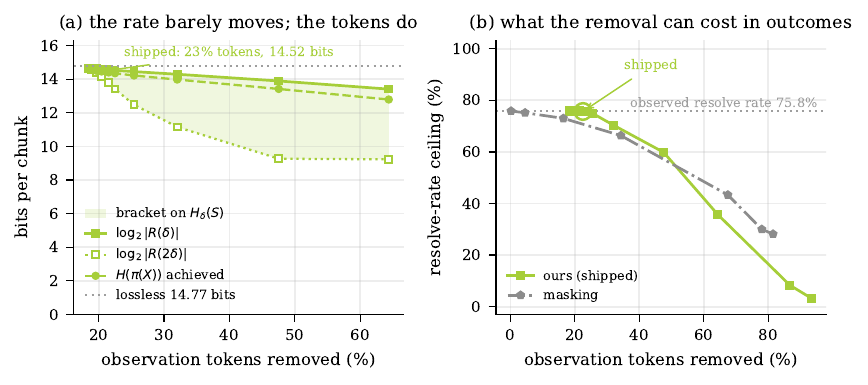}
\caption{\textbf{(a)} Metric-entropy bounds from Corollary~\ref{cor:bracket}
and observation-token removal. The upper and lower edges use greedy passes at
radii $\delta$ and $2\delta$. Index rate depends logarithmically on retained-set
size, while context cost depends on token count. \textbf{(b)} Conditional
resolve-rate ceiling (RSC, \eqref{eq:rsc}) versus token removal for LAM and
observation masking on \rduTraj{} trajectories, of which \rduResolved{}
were originally resolved. RSC remains near baseline through \rduTokNF{} removal
and decreases at more aggressive settings.}

\label{fig:rdu}
\end{figure}

\subsection{Rate, tokens, and the outcome ceiling}
\label{sec:rdu}

We measure both halves of Corollary~\ref{cor:rd} on \rduTraj{} trajectories over
a dense sweep of $\delta$. These were recorded \emph{without} compaction, so a
compacted rerun is unobservable, but each carries a \texttt{resolved} outcome
(\rduResolved{} of \rduTraj{}, rate \rscBase{}); crossing that with
per-trajectory evidence survival gives the \emph{resolve-rate ceiling}
\begin{equation}
\label{eq:rsc}
\mathrm{RSC}(\delta) \;\triangleq\;
  \frac{\left|\{\,\text{resolved trajectories retaining \emph{every} witnessed
  evidence line under $\delta$}\,\}\right|}{\text{\rduTraj{}}},
\end{equation}
which bounds live task success only if every witnessed line is necessary, lost
evidence cannot be reacquired, and compaction solves no previously unresolved
task; otherwise it is an evidence-based proxy. The harness checks that the sweep
is non-increasing and that $\mathrm{RSC}(0)$ matches the recorded resolve rate.

Against a lossless index of \rdLossless{} bits per chunk, at $\delta=0.3162$ the
bound gives $[\rdLoNF{},\rdHiNF{}]$ bits and the achieved entropy is \rdHNF{}, a
\rdBitSavePctNF{} reduction, while the same run removes \rdTokSaveNF{} of
observation tokens (Figure~\ref{fig:rdu}a) and RSC falls \rscLossNF{} percentage
points, from \rscBase{} to \rscNF{}. At similar removal masking with $k=20$
removes \rduTokMaskTwenty{} at RSC \rscMaskTwenty{} against \rduTokN{} and
\rscN{} for LAM at $\cos\ge0.90$; $k=50$ removes \rduTokMaskFifty{} at
\rscMaskFifty{}, and LAM at $\cos\ge0.80$ removes \rduTokEighty{} at
\rscEighty{} (Table~\ref{tab:rdu}). \citet{hou2026control} report a similarly
sharp task-success decrease under aggressive control-context compression, on
measurements not directly comparable with ours.

\subsection{Composite reduction: deduplication $\times$ byte codec}
\label{sec:composite}

Deduplication shrinks the number of vectors and the codec shrinks each survivor,
so the two multiply on the index axis but \emph{not} on the token axis: the
codec compresses embeddings, not the text the agent carries forward, so none of
these ratios is a context-length reduction. Crossing \cmpFront{}-plus merge
thresholds with the integrated byte codecs (\texttt{fp16}, per-vector
\texttt{int8}, sign-binary, cuSZp at several relative error bounds) gives
\cmpPoints{} combinations on one \cmpIndexN{}-vector index, scored against the
uncompressed, undeduplicated ground truth (recall@10 \cmpBaseRecall{}).

\begin{table}[t]
\centering
\caption{Composite storage reduction and recall@10, including the selected setting. Reduction multiplies deduplication and codec ratios on the embedding index. The full-precision, full-size index has recall \cmpBaseRecall{}; small differences around it should not be interpreted as quality gains.}
\label{tab:composite}
\small
\input{tables/composite}
\end{table}

\begin{table}[t]
\centering
\caption{Near-duplicate rate on LHTB, token-weighted, within trajectory, against the SWE-bench observation rate from Table~\ref{tab:ledger}. Same encoder, window and rule.}
\label{tab:lhtb}
\small
\input{tables/lhtb}
\end{table}

Composite storage reduction ranges from \cmpShipComposite{} to
\cmpMaxComposite{} and recall@10 from \cmpShipRecall{} to \cmpMaxRecall{}, but
at the most compressed setting retrieved-set Jaccard is \cmpMaxJaccard{}:
recall alone conceals a substantial change in which items come back. Because the
substitution bound does not cover byte-codec error we select
$\cos\ge\cmpShipCos{}$ with the conservative codec --- \cmpShipDedup{} $\times$
\cmpShipCodec{} $=$ \cmpShipComposite{} at recall \cmpShipRecall{} and Jaccard
\cmpShipJaccard{} --- over the higher-reduction \cmpKneeComposite{} at recall
\cmpKneeRecall{} and Jaccard \cmpKneeJaccard{}.

\subsection{SemDeDup matched on reach}
\label{sec:sdwithin}

The main SemDeDup baseline clusters the full corpus and so may match across
trajectories, while LAM searches only the current one. Rerunning it per
trajectory, with $k$ scaled to hold chunks per cluster fixed and the combined
deletion sets scored by the same criterion, removal falls to \sdwNineFiveRem{}
at $\cos t=0.95$ and \sdwNineNineRem{} at $0.99$, against \sdNineFiveRem{} and
\sdNineNineRem{} corpus-wide --- a factor of \sdwReachFactor{} at $0.95$.

At matched scope SemDeDup removes \sdwNineFiveRem{} retaining
\sdwNineFiveMicro{} of evidence, against \nearGreedyNineFive{} and
\safOursMicro{} for LAM's L1 alone, rising to \safOursRemoval{} once composed
with L0; SemDeDup could be composed with an exact layer too. Its
within-trajectory removal varies by \sdwSeedSpread{} across \sdSeeds{} seeds,
less than in the global experiment, though equal rates do not imply equal
deletion sets.

\subsection{A second corpus: long-horizon terminal agents}
\label{sec:lhtb}

\begin{figure}[t]
\centering
\includegraphics[width=0.50\linewidth]{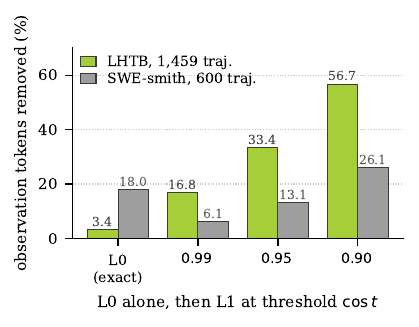}
\caption{The removal axis on a second scaffold. LHTB provides no gold patches, so only removal is plotted. L0 is record-level there and line-level on SWE-smith.}
\label{fig:ecr-lhtb}
\end{figure}

LHTB (Long-Horizon Terminal-Bench) contributes a different scaffold, task
distribution, and model set, with longer trajectories than SWE-smith:
\lhtbTraj{} trajectories from \lhtbModels{} models over \lhtbSteps{} steps
(median \lhtbStepsMed{}, 90th percentile \lhtbStepsPNinety{}, maximum
\lhtbStepsMax{}) and \lhtbPromptTok{} cumulative prompt tokens. The same
encoder, 512-token windows, and arrival-order greedy rule run over
\lhtbEmbTraj{} trajectories: \lhtbWindows{} windows, \lhtbWindowsObs{} of them
observations, \lhtbTokObs{} observation tokens. Each step carries an agent
message, a tool call, and an observation, so Table~\ref{tab:lhtb} separates
observation-only from whole-context removal.

At $\cos t=0.95$ LAM removes \lhtbNearNineFiveObs{} of observation tokens
against \nearGreedyNineFive{} on SWE-smith, a factor of \lhtbRatioNineFive{},
and at $0.99$ \lhtbNearNineNineObs{} against \nearGreedyNineNine{}, a factor of
\lhtbRatioNineNine{}; whole-context removal is close to observation-only at
$0.95$ (\lhtbNearNineFiveAll{} and \lhtbNearNineFiveObs{}). Exact record
duplication accounts for \lhtbExactObs{}, at a different granularity from the
line-level L0 hashing reported on SWE-smith. The experiment measures redundancy, not task
success under compaction.

\subsection{Band width and quantizer step: the full grid}
\label{sec:banding}

We sweep band width $r\in\bandRs{}$ and quantizer step
$\varepsilon=\mathrm{rel}\cdot\sigma$ for $\mathrm{rel}\in\bandRels{}$ on
\bandRecords{} observation records from \bandTraj{} trajectories (\bandPairs{}
within-trajectory pairs). Recall against the exhaustive $\cos\ge t$ pair set has
to be full to reproduce exhaustive greedy removal; lower recall still preserves
the substitution bound.

At the default $\cos\ge0.95$, full recall requires $r=\bandShipR{}$ and
$\mathrm{rel}=\bandShipRel{}$, examining \bandShipCand{} of pairs for a
\bandShipSpeedup{} reduction, while accepting recall \bandRelaxRecall{} at
$r=\bandRelaxR{}$, $\mathrm{rel}=\bandRelaxRel{}$ gives \bandRelaxSpeedup{}.
Since a band match requires agreement in all its coordinates, every threshold
needs its own setting: the best $\cos\ge0.99$ configuration is $r=\bandBestR{}$,
$\mathrm{rel}=\bandBestRel{}$ at \bandBestCand{} of pairs (\bandBestSpeedup{},
precision \bandBestPrec{}). Appendix~\ref{app:extra} gives the grid.

\subsection{Brute-force scan crossover}
\label{sec:index-scaling}

Exact search costs scale linearly in retained embeddings, and within-trajectory
indexing confines it to one agent's history: replaying \encSensTraj{}
trajectories gives median \ixSurvMed{} survivors and maximum \ixSurvMax{},
occupying \ixReachedMem{}, with search taking \ixSearchUs{} of the \ixAdmitMs{}
call (\ixSearchFrac{} of elapsed time).

A synthetic index locates the crossover with encoding cost. Exact search is a
matrix--vector product: at a million survivors it reaches \ixEffBw{},
\ixBwOfPeak{} of measured triad bandwidth (Appendix~\ref{app:anchors}), costing
\ixSearchPerK{} per thousand survivors, and matches encoding cost only at
\ixCrossover{} survivors --- \ixCrossoverRatio{} the largest observed index and
\ixCrossoverMem{} of embedding storage. Bounding index growth therefore keeps
these workloads clear of approximate search.

%% file: tables/delta_hat.tex
\begin{tabular}{crrrrrrr}
\toprule
& \multicolumn{5}{c}{substitution distance $d_i$, token-weighted} & \multicolumn{2}{c}{retrieval} \\
\cmidrule(lr){2-6}\cmidrule(lr){7-8}
$\cos t$ & bound $\delta$ & median & p99 & max & mean$/\delta$ & keep & R@10 \\
\midrule
$0.99$ & 0.1414 & 0.0376 & 0.1401 & 0.1412 & 0.350 & 0.916 & 0.9978 \\
$0.98$ & 0.2000 & 0.0820 & 0.1979 & 0.1999 & 0.409 & 0.896 & 0.9947 \\
$0.95$ & 0.3162 & 0.2103 & 0.3148 & 0.3162 & 0.559 & 0.837 & 0.9846 \\
$0.9$ & 0.4472 & 0.3598 & 0.4457 & 0.4472 & 0.706 & 0.696 & 0.9518 \\
\bottomrule
\end{tabular}

%% file: tables/composite.tex
\begin{tabular}{lrrrrr}
\toprule
$\cos t$ & byte codec & dedup & codec & composite & recall@10 \\
\midrule
$0.9$ & binary (sign) & $1.396\times$ & $30.72\times$ & $42.88\times$ & 0.6204 \\
$0.9$ & cuSZp/fixed, $\varepsilon_b\!=\!0.03$ & $1.396\times$ & $9.46\times$ & $13.20\times$ & 0.6354 \\
$0.9$ & cuSZp/fixed, $\varepsilon_b\!=\!0.01$ & $1.396\times$ & $6.53\times$ & $9.11\times$ & 0.6491 \\
$0.9$ & cuSZp/fixed, $\varepsilon_b\!=\!0.003$ & $1.396\times$ & $4.94\times$ & $6.90\times$ & 0.6549 \\
$0.9$ & int8 (per-vector) & $1.396\times$ & $3.96\times$ & $5.53\times$ & 0.6556 \\
\midrule
\multicolumn{6}{l}{\emph{operating point we ship}} \\
$0.95$ & cuSZp/fixed, $\varepsilon_b\!=\!0.001$ & $1.188\times$ & $3.90\times$ & $4.63\times$ & 0.6543 \\
\bottomrule
\end{tabular}

%% file: tables/lhtb.tex
\begin{tabular}{lrrrr}
\toprule
& \multicolumn{2}{c}{LHTB} & \multicolumn{2}{c}{SWE-bench} \\
\cmidrule(lr){2-3}\cmidrule(lr){4-5}
$\cos t$ & all context & observations & observations & ratio \\
\midrule
L0 (byte-exact) & 2.08\% & 3.35\% & --- & --- \\
$0.999$ & 6.28\% & 8.07\% & --- & --- \\
$0.99$  & 14.44\% & 16.78\% & 6.52\% & $2.6\times$ \\
$0.98$  & 18.71\% & 20.54\% & 8.39\% & $2.4\times$ \\
$0.95$  & 33.07\% & 33.39\% & 13.81\% & $2.4\times$ \\
$0.9$   & 58.72\% & 56.72\% & 27.32\% & $2.1\times$ \\
\bottomrule
\end{tabular}

%% file: sections/A-appendix.tex
\section{Measured hardware constants}
\label{app:anchors}

All constants are measured on a single GB10 (Grace Blackwell, 48 SM). Eight of
the nine rows are kernel microbenchmarks at sizes large enough to reach steady
state; only $\beta$ comes from real data.

\begin{table}[h]
\centering
\caption{Hardware anchors and the benchmark used to obtain each.}
\label{tab:anchors}
\footnotesize
\begin{tabular}{llp{7.2cm}}
\toprule
quantity & measured & benchmark \\
\midrule
streaming bandwidth & \hwBW & triad $c = a+b$, $3\times512$\,MB fp32; $\mu_m = \hwMUm$ of the 273\,GB/s spec \\
fp16 tensor matmul  & \hwFLOPS & $8192^3$ GEMM \\
machine balance     & \hwBalance\,FLOP/byte & derived $F/B$; everything below is bandwidth-bound \\
bge-base encode     & \encRate\,tok/s & 512 \emph{synthetic} docs $\times$ 512 tok = 262K tok, batch 64 \\
$Q_\varepsilon$ quantize & \quantRate & $2^{26}$ fp32 $\to$ int16, 6\,B/element moved; unfused PyTorch \\
xxHash3 (CPU)       & \hashRate & 64\,MB buffer, single thread \\
exact $\ell_2$ verify & \lTwoRate\,M pairs/s & 65{,}536 candidate pairs at $D = 768$ \\
cuSZp               & \cuszpRate & $2^{24}$ fp32, error bound 1\% of $\max|x|$; ratio \cuszpCR$\times$ \\
bytes per token $\beta$ & \bytesPerTok\,B/tok & SWE-smith \texttt{tool} split, 200 trajectories, observation records \\
\bottomrule
\end{tabular}
\end{table}

\section{Compaction latency on a second device}

\begin{figure}[h]
\centering
\includegraphics[width=0.8\linewidth]{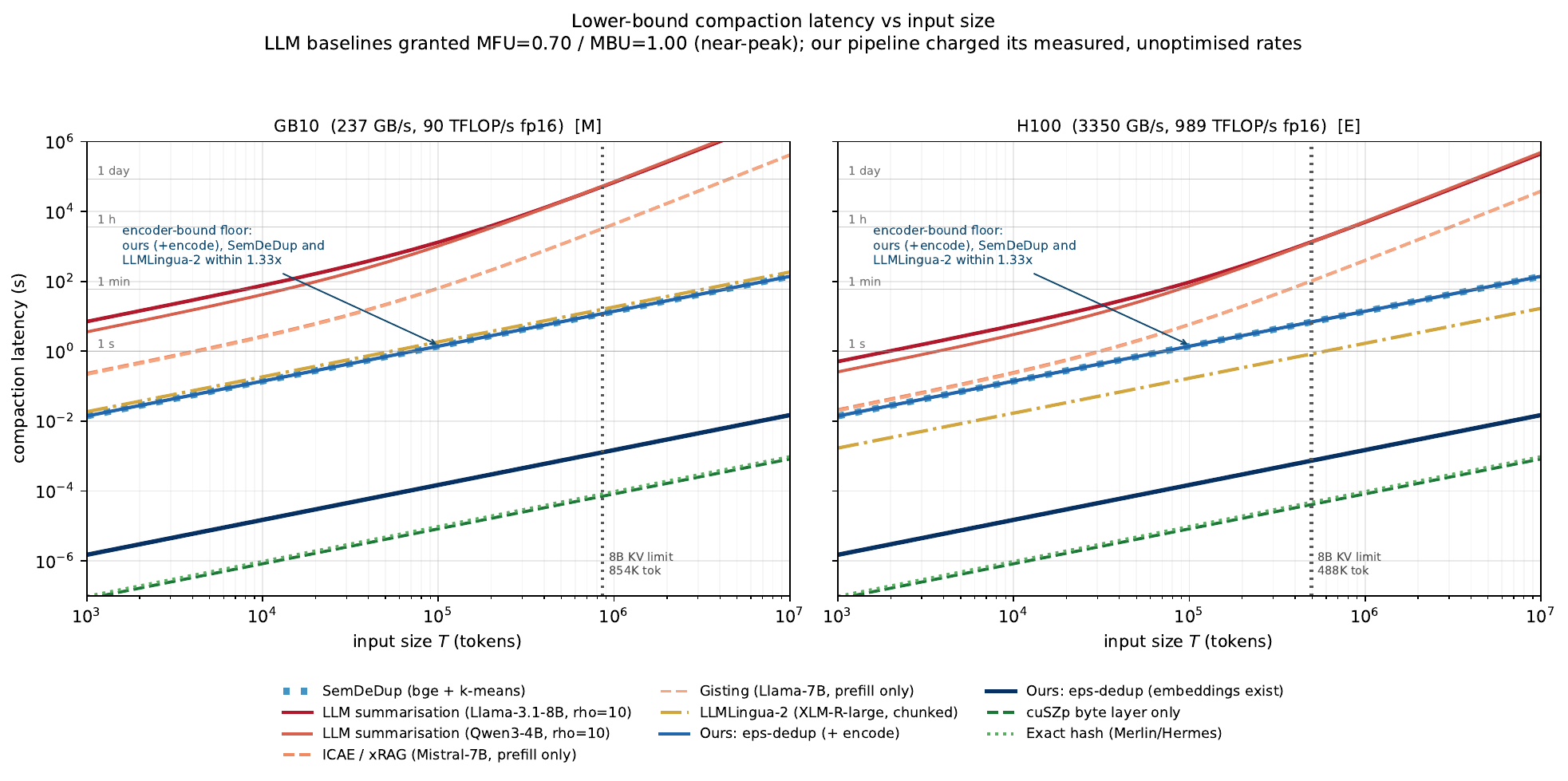}
\caption{Compaction latency on a second device (\xhwHwB{}), the same comparison as
Figure~\ref{fig:main}(a). The ordering and the crossovers are unchanged, so the
gap is structural rather than an artifact of one bandwidth-starved machine.
Dotted verticals mark where the KV cache no longer fits at batch 1.}
\label{fig:latency}
\end{figure}

\section{Deferred proofs}
\label{app:proofs}

\subsection*{Corollary~\ref{cor:rd} (distribution-free rate--distortion bound)}

\begin{proof}
Let $(R,\pi)$ be a $\delta$-substitution of $S$ and let $X \sim P$ be supported
on $S$. Because $\pi$ is a deterministic function of $X$ we have
$H(\pi(X) \mid X) = 0$, so
\[
I\!\left(X; \pi(X)\right) \;=\; H\!\left(\pi(X)\right) - H\!\left(\pi(X) \mid X\right)
\;=\; H\!\left(\pi(X)\right) \;\le\; \log_2 |R|,
\]
the last step because $\pi(X)$ takes at most $|R|$ values. By
Definition~\ref{def:sub}, $\|x - \pi(x)\|_2 \le \delta$ for \emph{every}
$x \in S$, hence $\|X - \pi(X)\|_2 \le \delta$ almost surely and, taking
expectations, $\mathbb{E}\|X - \pi(X)\|_2 \le \delta$. The joint law of
$(X, \pi(X))$ is therefore feasible for the rate--distortion problem at
distortion $\delta$, and since $R_P(\delta)$ is the infimum of $I(X;\hat X)$ over
feasible couplings, $R_P(\delta) \le I(X;\pi(X))$. The almost-sure variant
restricts the feasible set to couplings satisfying the constraint pointwise,
which $(X,\pi(X))$ does.
\end{proof}

The bound holds for every $P$ supported on $S$, and claims tightness for none.
The cardinality bound $\log_2|R|$ equals the achieved entropy $H(\pi(X))$ when
retained clusters carry equal probability mass; their difference in
Table~\ref{tab:rdu} is the empirical imbalance.

\subsection*{Corollary~\ref{cor:bracket} (measured bracket)}

\begin{proof}
We take the four terms of \eqref{eq:bracket} left to right.

\emph{$\log_2|R_{\mathrm{greedy}}(2\delta)| \le \sup_P R_P^{\mathrm{as}}(\delta)$.}
Write $P^\star = R_{\mathrm{greedy}}(2\delta)$; by the packing half of
Proposition~\ref{prop:pack} at radius $2\delta$, its elements are pairwise
separated by more than $2\delta$. Let $P$ be uniform on $P^\star$ and let
$\hat X$ be any reconstruction with $\|X - \hat X\|_2 \le \delta$ almost surely.
If distinct $p, p' \in P^\star$ were assigned a common $\hat x$, the triangle
inequality would give $\|p - p'\|_2 \le 2\delta$, contradicting the separation.
So $\hat X$ determines $X$, giving $I(X;\hat X) = H(X) = \log_2|P^\star|$.

\emph{$\sup_P R_P^{\mathrm{as}}(\delta) \le H_\delta(S)$.} Let $C$ be a minimal
internal $\delta$-covering, $|C| = N(S,\delta)$, and for any $P$ take $\hat X$ to
be a nearest element of $C$ to $X$, which satisfies the constraint by definition.
Then $R_P^{\mathrm{as}}(\delta) \le I(X;\hat X) \le H(\hat X) \le \log_2 |C|$.

\emph{$H_\delta(S) \le \log_2|R_{\mathrm{greedy}}(\delta)|$.} By the covering
half of Proposition~\ref{prop:pack}, $R_{\mathrm{greedy}}(\delta)$ is an internal
$\delta$-covering of $S$, so its cardinality is at least $N(S,\delta)$.
\end{proof}

The middle two quantities need not be equal: the lower bound uses a
$2\delta$-packing and the upper bound a $\delta$-covering, with
$M(S,2\delta)\le N(S,\delta)\le M(S,\delta)$. Two greedy passes then give
computable bounds.

\section{The dense $\delta$ sweep}
\label{app:sweep}

Table~\ref{tab:rdu} is the full sweep behind Figure~\ref{fig:rdu}. The outer bit
columns are the bracket of Corollary~\ref{cor:bracket}, greedy first-occurrence
at radius $2\delta$ and at $\delta$. The middle column is not part of it:
$H(\pi(X))$ is the rate the $\delta$ run \emph{achieves} on the empirical
source, which \eqref{eq:rd} bounds only above by $\log_2|R(\delta)|$ --- every
row satisfies that, and at $t = 0.995$ it falls just under the $2\delta$ column,
which is permitted because the bracket constrains the supremum over sources
rather than one empirical source. A dash in the $2\delta$ column marks a
threshold whose partner $t' = 4t-3$ falls below $0.2$, where that run merges
nearly everything and the ceiling is vacuous. Masking is not a
$\delta$-substitution, so its rows carry no bit columns and appear only for the
matched-removal comparison. RSC is the conditional ceiling in \eqref{eq:rsc},
not measured task success.

\begin{table}[h]
\centering
\caption{Rate, removal and the outcome ceiling across the sweep. Bits per
chunk; $\log_2|S| = \rdLossless$ lossless.}
\label{tab:rdu}
\small
\input{tables/rdu_sweep}
\end{table}

\section{Additional measurements}
\label{app:extra}

\subsection{Rank invariance, in full}
\label{app:margin}

We evaluate ranking margins on the LoCoMo retrieval store --- \margRec{}
dialogue-turn records and \margQ{} questions with gold evidence, excluding the
unanswerable split, under \texttt{bge-base-en-v1.5} at $D=\margDim{}$ --- which
isolates the margin requirement from SWE-smith evidence retention. For unit
queries a coordinatewise tolerance gives $|\Delta\cos|\le2\varepsilon\sqrt{D}$,
while distance verification gives $|\Delta\cos|=|q\cdot(a-b)|\le\|a-b\|_2\le
\sqrt{2-2t}$ at merge threshold $t$, linking the substitution bound to the
merge threshold without a $\sqrt{D}$ factor.

\paragraph{Invariance rates.}
At $\varepsilon=\margCertBestEps{}$ the coordinatewise bound proves invariance
for \margCertBest{} of queries at $k=1$ and \margCertBestThree{} at $k=3$; the
$\ell_2$ bound proves \margCertNNF{} at $k=1$ in its tightest tested setting and
\margCertNN{} at $\cos\ge0.99$. At the default $\cos\ge0.95$ no tested $k$ is
guaranteed invariant: median adjacent-rank margins are \margMedOne{},
\margMedThree{}, and \margMedTwenty{} at $k=1,3,20$ against bounds
\margBoundNN{} and \margBoundNF{}. Margins this tight are why
Section~\ref{sec:theory} states the guarantee on scores rather than ranks.

\begin{table}[t]
\centering
\caption{Share of queries whose top-$k$ is provably invariant, by bound and by
$k$. \margQ{} queries against \margRec{} LoCoMo records. The default operating
point is $t = 0.95$.}
\label{tab:margin-cert}
\small
\input{tables/margin_cert}
\end{table}

\begin{table}[t]
\centering
\caption{Rank-$k$/rank-$(k{+}1)$ cosine margin over \margQ{} queries. Compare
against the bound column of Table~\ref{tab:margin-cert}.}
\label{tab:margin-quantiles}
\small
\input{tables/margin_quantiles}
\end{table}

\subsection{Band-width study}

Table~\ref{tab:band-sweep} is the grid behind Section~\ref{sec:banding}, over
\bandTraj{} trajectories (\bandRecords{} records, \bandPairs{} within-trajectory
pairs), with $\sigma = \bandSigma{}$ the RMS coordinate magnitude of the
encoder's output. Recall is over the true pairs at the stated cosine, so $1.000$
means the banded index proposed every pair exhaustive comparison would merge.

\begin{table}[t]
\centering
\caption{Band-width sweep over the exhaustive within-trajectory pair set.
Candidate rate is the fraction of pairs sent to distance verification; its
reciprocal gives the reduction in comparisons. We select settings with full
recall and low candidate rate.}
\label{tab:band-sweep}
\small
\input{tables/band_sweep}
\end{table}

The grid trades pair recall against candidate count. At $r=2$, full recall at
$\cos\ge0.99$ costs \bandNarrowCand{} candidates with
$\varepsilon/\sigma=\bandNarrowRel{}$; the cheapest full-recall setting is
$r=\bandBestR{}$, $\varepsilon/\sigma=\bandBestRel{}$, examining \bandBestCand{}
of pairs for a \bandBestSpeedup{} reduction; at $r=32$ recall never exceeds
\bandWideRecallNineNine{} at $\cos\ge0.99$ or \bandWideRecallNineFive{} at
$\cos\ge0.95$. Settings are specific to the threshold: the best $0.99$
configuration reaches only $0.651$ recall at $0.95$.

\section{Design choices and the scope}
\label{app:corrections}

These choices connect the implementation to its guarantees and cost model.

\begin{enumerate}\itemsep2pt
\item \textbf{Direct $\ell_2$ verification}
(Section~\ref{sec:method}). For unit-normalized embeddings cosine maps directly
to substitution distance. Quantization proposes candidates; the distance check
decides the merge.

\item \textbf{A query-independent substitution bound}
(Appendix~\ref{app:margin}). The theorem bounds maximum retrieval-score loss
with no ranking-margin assumption; ranking quality and evidence retention are
measured separately.

\item \textbf{Retained representatives}
(Section~\ref{sec:rule}). Every deleted record lies within $\delta$ of its
representative. Greedy first-occurrence enforces that without assuming the
merge relation is transitive.

\item \textbf{Per-arrival accounting}
(Section~\ref{sec:arrival}). Removal is charged once per arriving token, the
context grows by the retained arrivals, and a hard size limit needs eviction on
top of deduplication.

\item \textbf{Explicit reproducibility conditions}
(Section~\ref{sec:determinism}). LAM is deterministic for fixed embeddings,
arrival order, and arithmetic; cross-batch and cross-hardware stability are
measured separately. The LLM deletion baseline reproduces only under its tested
serving configuration.
\end{enumerate}

%% file: tables/rdu_sweep.tex
\small
\begin{tabular}{llcccccc}
\toprule
& & \multicolumn{3}{c}{bits per chunk} & & & \\
\cmidrule(lr){3-5}
$t$ & $\delta$ & $\log_2|R(2\delta)|$ & $H(\pi(X))$ & $\log_2|R(\delta)|$
    & tokens rm.\ & traj.\ intact & RSC \\
\midrule
$0.995$ & $0.1000$ & 14.61 & 14.59 & 14.67 & 18.47\% & 100.0\% & 75.83\% \\
$0.990$ & $0.1414$ & 14.55 & 14.56 & 14.65 & 18.89\% & 100.0\% & 75.83\% \\
$0.980$ & $0.2000$ & 14.40 & 14.50 & 14.61 & 19.64\% & 100.0\% & 75.83\% \\
$0.970$ & $0.2449$ & 14.14 & 14.45 & 14.58 & 20.45\% & 100.0\% & 75.83\% \\
$0.960$ & $0.2828$ & 13.80 & 14.40 & 14.55 & 21.53\% & 100.0\% & 75.83\% \\
$0.950$ & $0.3162$ & 13.41 & 14.35 & 14.52 & 22.50\% & 99.8\% & 75.67\% \\
$0.930$ & $0.3742$ & 12.48 & 14.22 & 14.45 & 25.40\% & 99.0\% & 75.00\% \\
$0.900$ & $0.4472$ & 11.15 & 13.98 & 14.29 & 32.02\% & 93.0\% & 70.33\% \\
$0.850$ & $0.5477$ & 9.26 & 13.42 & 13.89 & 47.55\% & 76.9\% & 59.83\% \\
$0.800$ & $0.6325$ & 9.23 & 12.79 & 13.41 & 64.38\% & 43.1\% & 35.67\% \\
$0.700$ & $0.7746$ & --- & 11.39 & 12.20 & 86.71\% & 9.5\% & 8.33\% \\
$0.600$ & $0.8944$ & --- & 10.21 & 11.15 & 93.31\% & 3.3\% & 3.33\% \\
\midrule
mask $k{=}100$ & --- & --- & --- & --- & 0.20\% & 99.8\% & 75.83\% \\
mask $k{=}50$ & --- & --- & --- & --- & 4.56\% & 98.3\% & 75.17\% \\
mask $k{=}30$ & --- & --- & --- & --- & 16.43\% & 92.5\% & 73.00\% \\
mask $k{=}20$ & --- & --- & --- & --- & 34.28\% & 81.9\% & 66.33\% \\
mask $k{=}10$ & --- & --- & --- & --- & 67.47\% & 52.2\% & 43.33\% \\
mask $k{=}5$ & --- & --- & --- & --- & 78.00\% & 35.6\% & 30.00\% \\
mask $k{=}3$ & --- & --- & --- & --- & 81.49\% & 32.8\% & 28.17\% \\
\bottomrule
\end{tabular}

%% file: tables/margin_cert.tex
\small
\begin{tabular}{llrrrrrr}
\toprule
& & & \multicolumn{5}{c}{queries with bounded-invariant top-$k$} \\
\cmidrule(lr){4-8}
relation & parameter & bound on $|\Delta\cos|$
  & $k{=}1$ & $k{=}3$ & $k{=}5$ & $k{=}10$ & $k{=}20$ \\
\midrule
\multirow{4}{*}{$\ell_\infty$ relaxed, $2\varepsilon\sqrt{D}$}
  & $\varepsilon = 0.001$ & 0.0554 & 15.62\% & 0.59\% & 0.13\% & 0.00\% & 0.00\% \\
  & $\varepsilon = 0.003$ & 0.1663 &  0.26\% & 0.00\% & 0.00\% & 0.00\% & 0.00\% \\
  & $\varepsilon = 0.01$  & 0.5543 &  0.00\% & 0.00\% & 0.00\% & 0.00\% & 0.00\% \\
  & $\varepsilon = 0.03$  & 1.6628 &  0.00\% & 0.00\% & 0.00\% & 0.00\% & 0.00\% \\
\midrule
\multirow{6}{*}{$\ell_2$ tight, $\sqrt{2-2t}$}
  & $t = 0.995$ & 0.1000 & 3.71\% & 0.00\% & 0.00\% & 0.00\% & 0.00\% \\
  & $t = 0.99$  & 0.1414 & 0.46\% & 0.00\% & 0.00\% & 0.00\% & 0.00\% \\
  & $t = 0.98$  & 0.2000 & 0.07\% & 0.00\% & 0.00\% & 0.00\% & 0.00\% \\
  & $t = 0.95$  & 0.3162 & 0.00\% & 0.00\% & 0.00\% & 0.00\% & 0.00\% \\
  & $t = 0.90$  & 0.4472 & 0.00\% & 0.00\% & 0.00\% & 0.00\% & 0.00\% \\
  & $t = 0.85$  & 0.5477 & 0.00\% & 0.00\% & 0.00\% & 0.00\% & 0.00\% \\
\bottomrule
\end{tabular}

%% file: tables/margin_quantiles.tex
\begin{tabular}{crrrr}
\toprule
$k$ & median & p90 & p99 & max \\
\midrule
$1$  & 0.0185 & 0.0702 & 0.1296 & 0.2358 \\
$3$  & 0.0056 & 0.0201 & 0.0483 & 0.0886 \\
$5$  & 0.0036 & 0.0133 & 0.0299 & 0.0698 \\
$10$ & 0.0018 & 0.0067 & 0.0148 & 0.0247 \\
$20$ & 0.0010 & 0.0034 & 0.0075 & 0.0147 \\
\bottomrule
\end{tabular}

%% file: tables/band_sweep.tex
\begin{tabular}{rrrrrrrrr}
\toprule
& \multicolumn{8}{c}{$\varepsilon/\sigma$} \\
\cmidrule(lr){2-9}
$r$ & $0.01$ & $0.02$ & $0.03$ & $0.05$ & $0.08$ & $0.12$ & $0.2$ & $0.3$ \\
\midrule
\multicolumn{9}{l}{\emph{recall of true $\cos \ge 0.99$ pairs}} \\
$2$ & 0.974 & 1.000 & 1.000 & 1.000 & 1.000 & 1.000 & 1.000 & 1.000 \\
$4$ & 0.581 & 0.774 & 0.885 & 0.990 & 1.000 & 1.000 & 1.000 & 1.000 \\
$8$ & 0.400 & 0.506 & 0.590 & 0.743 & 0.879 & 0.986 & 1.000 & 1.000 \\
$16$ & 0.332 & 0.398 & 0.467 & 0.537 & 0.668 & 0.779 & 0.934 & 0.992 \\
$32$ & 0.323 & 0.333 & 0.358 & 0.414 & 0.486 & 0.550 & 0.685 & 0.769 \\
\midrule
\multicolumn{9}{l}{\emph{candidate pairs, \% of all within-trajectory pairs}} \\
$2$ & 5.40 & 17.36 & 33.19 & 64.21 & 90.95 & 99.18 & 100.00 & 100.00 \\
$4$ & 0.48 & 0.66 & 0.79 & 1.19 & 2.61 & 7.54 & 31.65 & 73.06 \\
$8$ & 0.33 & 0.42 & 0.49 & 0.61 & 0.73 & 0.90 & 1.46 & 3.84 \\
$16$ & 0.27 & 0.33 & 0.39 & 0.44 & 0.55 & 0.64 & 0.79 & 0.98 \\
$32$ & 0.27 & 0.28 & 0.30 & 0.34 & 0.40 & 0.45 & 0.57 & 0.64 \\
\midrule
\multicolumn{9}{l}{\emph{recall of true $\cos \ge 0.95$ pairs}} \\
$2$ & 0.574 & 0.844 & 0.955 & 1.000 & 1.000 & 1.000 & 1.000 & 1.000 \\
$4$ & 0.237 & 0.322 & 0.378 & 0.488 & 0.672 & 0.931 & 1.000 & 1.000 \\
$8$ & 0.163 & 0.206 & 0.240 & 0.303 & 0.361 & 0.441 & 0.651 & 0.979 \\
$16$ & 0.135 & 0.162 & 0.191 & 0.219 & 0.272 & 0.318 & 0.391 & 0.481 \\
$32$ & 0.132 & 0.136 & 0.146 & 0.169 & 0.198 & 0.224 & 0.279 & 0.314 \\
\bottomrule
\end{tabular}

%% file: sections/I-cost.tex
\section{Cost accounting and compaction latency}
\label{app:costclasses}

\subsection{Three cost classes}

A wall-clock comparison has to charge each method for what it does, which
separates them into three classes.

\textbf{Class 1: decode-bound LLM} (summarization). With input $T$ and ratio
$\rho$, the method decodes $G = T/\rho$ tokens:
\begin{equation}
t = \underbrace{\frac{2PT + 2LdT^2}{\mu_c F}}_{\text{prefill}}
  + \underbrace{\frac{G P b_w + 2 L H_{kv} d_h b_{kv}(GT + G^2/2)}{\mu_m B}}_{\text{decode}}
  + G\tau .
\end{equation}
The term $G P b_w / (\mu_m B)$ is linear in $T$ and \emph{irreducible by
parallelism}, since decode is sequential; on our hardware an 8B model in fp16
costs 67\,ms per token from this term alone.

\textbf{Class 2: prefill-only LLM} (Gisting, ICAE, xRAG, LLMLingua-2). One
forward pass emits the representation, so there is no decode term. A causal
backbone gives $t = (2PT + 2LdT^2)/(\mu_c F)$, quadratic in $T$; a \emph{chunked
encoder} windows attention for $t = (2PT + 4LdTW)/(\mu_c F)$, linear.
LLMLingua-2 is in the second group, so it is the cheaper LLM baseline.

\textbf{Class 3: no LLM} (all of LAM's layers, and byte codecs generally):
\begin{equation}
t = \frac{T}{R_{enc}} + \frac{6ND}{B_q} + \frac{2ND}{R_{hash}}
  + \frac{n_{cand}}{R_{\ell_2}} + \frac{\beta T(1 + 1/\mathrm{CR})}{R_{cuSZp}} .
\end{equation}
Every stage is below the machine balance (\hwBalance{}\,FLOP/byte, measured), so
all are bandwidth-bound.

\paragraph{Fairness rule.} LLM baselines get near-peak efficiency ($\mu_c =
0.70$, $\mu_m = 1.00$) while our pipeline is charged its measured, unoptimized
rates, so the comparison favors the baselines. These are modeling assumptions,
not runtime bounds; Appendix~\ref{app:anchors} lists the hardware measurements.

\subsection{Compaction latency and its limits as an argument}

\begin{figure}[t]
\centering
\includegraphics[width=0.84\linewidth]{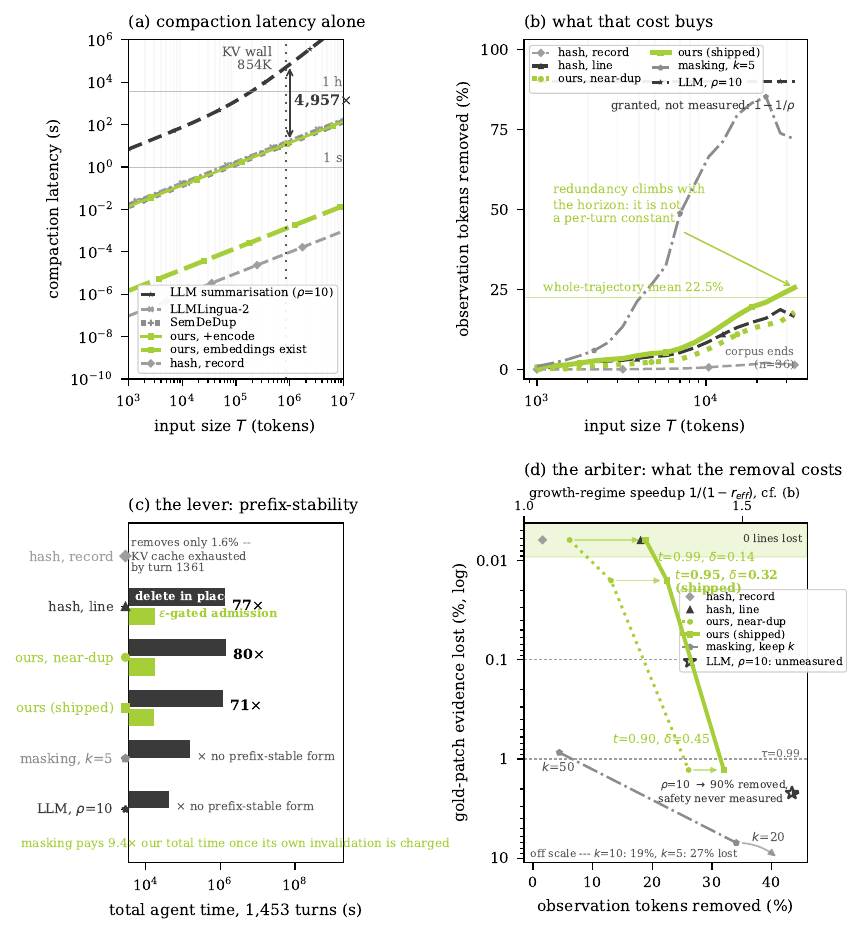}
\caption{Compaction cost, removal, scheduling, and evidence retention.
\textbf{(a)} Modeled compaction latency on $T$ tokens, excluding subsequent
agent inference. \textbf{(b)} Observation-token removal versus input size over
\growthCurveTraj{} trajectories. Removal rises from \growthSmallRem{} at
$T=\growthSmallT{}$ to \growthLargeRem{} at $T=\growthLargeT{}$, with
whole-trajectory mean \growthWholeRem{}. Curves stop when too few trajectories
reach the input size; summarization removal is assumed at $1-1/\rho$.
\textbf{(c)} Scheduling at fixed compressor, corpus, and removal.
Methods that rewrite or delete existing context pay re-prefill costs;
observation masking costs \maskOverOurs{} times LAM's modeled total time.
For rules supporting both schedules, testing first gives
\polGainOne{}--\polGainBs{} speedup.
\textbf{(d)} Gold-patch evidence loss on an inverted logarithmic axis.
Zero-loss methods appear in the shaded band. The plot is clipped at
\lossClip{} loss. The default setting loses \safOursShipLines{} of
\safLines{} lines (\safOursShipLoss{}), compared with masking losses of
\safMaskTwentyLoss{}, \safMaskTenLoss{}, and \safMaskFiveLoss{} at $k=20,10,5$.
The $\tau$ floors and \lossIntolerable{} band mark retention requirements.
Summarization evidence retention is unmeasured.}

\label{fig:main}
\end{figure}

Figure~\ref{fig:main}(a) compares the three classes and
Figure~\ref{fig:latency} repeats it on a second device; vertical lines mark the
modeled KV capacity at batch 1 (\kvCap{} tokens for Llama-3.1-8B on GB10),
beyond which the unmodified full-context configuration does not fit. At
$T=10^6$ the modeled compaction speedups are \compVsLlama{} over summarization,
\compVsLingua{} over LLMLingua-2, and \compVsSemDeDup{} over SemDeDup. The
LLMLingua-2 cost model uses XLM-R-large; the evidence experiment in
Appendix~\ref{sec:relaxed} uses the smaller \texttt{bert-base-multilingual}
variant configured by LightMem.

Encoding is \encodeShare{} of LAM's cost at $T=10^6$ and SemDeDup uses the same
encoder, so the two cost about the same at trajectory scale. Where a memory
store already holds the embeddings, LAM costs \oursNoEncode{} against
\hashFloor{} for L0 alone, a factor of \vsHashFloor{}. We report both cases
because reuse depends on the deployment.

\subsection{Per-arrival removal and context growth}
\label{sec:arrival}

The measured removal rate $r$ is the fraction of arriving tokens duplicating
earlier content. A token can be removed once, so context grows as $C\leftarrow
C+\Delta(1-r)$ per arrival; applying $r$ to the full context on every event
would count the same duplicates again. Testing first imposes no steady-state
size on its own, so a hard budget also needs eviction
(Table~\ref{tab:arrival}).

\begin{table}[t]
\centering
\caption{Snapshot and per-arrival accounting for the same measured removal rate.}
\label{tab:arrival}
\footnotesize
\begin{tabular}{lll}
\toprule
 & snapshot assumption & per-arrival (correct) \\
\midrule
update rule          & $C \leftarrow C(1-r)$ every event & $C \leftarrow C + \Delta(1-r)$ per arrival \\
steady state         & $C^{*} = \Delta / r_{\mathit{eff}}$ & none; only eviction bounds $C$ \\
mean context, 32K budget & \meanCtxSnapshot{} & \meanCtxArrival{} \\
speedup of L1 over L0 alone & 1.22$\times$ & \growthGain{} growth, \budgetGain{} hard budget \\
\bottomrule
\end{tabular}
\end{table}

The two accountings diverge where it matters. In the growth regime LAM supports
\growthTurnsOurs{} turns before the KV capacity limit against
\growthTurnsLine{} for L0 alone, a modeled \growthGain{} speedup. Under a hard
budget every method eventually reaches $B_{ctx}$, so additional removal gives
\budgetGain{} latency benefit; it buys delayed eviction of nonredundant content
instead, a quality effect the model does not capture.